\documentclass{article}

\usepackage[preprint]{neurips_2025}

\usepackage[utf8]{inputenc}
\usepackage[T1]{fontenc}
\usepackage{hyperref}
\usepackage{url}
\usepackage{booktabs}
\usepackage{amsfonts}
\usepackage{amsmath, amssymb, amsthm}
\usepackage{nicefrac}
\usepackage{microtype}
\usepackage[table]{xcolor}
\usepackage{multirow}
\usepackage{graphicx}
\usepackage{enumitem}
\usepackage{wrapfig}

\newtheorem{theorem}{Theorem}
\newtheorem{lemma}[theorem]{Lemma}
\newtheorem{assumption}[theorem]{Assumption}
\newtheorem{proposition}[theorem]{Proposition}
\newtheorem{corollary}[theorem]{Corollary}

\newtheorem{definition}[theorem]{Definition}

\definecolor{dpoBoxFill}{HTML}{F7EBEB}
\definecolor{dpoRedText}{HTML}{CC6A6A}
\definecolor{dpoHighlight}{HTML}{F2E2E2}

\definecolor{simpoBoxFill}{HTML}{E8EFF9}
\definecolor{simpoBlueText}{HTML}{5A82C7}
\definecolor{simpoHighlight}{HTML}{E5E9F3}

\definecolor{teal1}{RGB}{77,175,159}
\definecolor{blue1}{RGB}{146, 181, 202}
\definecolor{red1}{RGB}{230, 145, 145}
\definecolor{orange1}{RGB}{255,158,74}
\definecolor{purple1}{RGB}{158,154,200}
\definecolor{green1}{RGB}{107,142,35}
\definecolor{deepblue}{RGB}{89, 156, 180}
\definecolor{deepred}{RGB}{194, 87, 89}
\definecolor{warmred}{RGB}{196,78,82}

\title{Kernel Selection is Model Selection:\\A Unified Complexity-Penalized Approach for MMD Two-Sample Tests}

\author{%
  Yijin Ni \\
  H.~Milton Stewart School of\\Industrial and Systems Engineering \\
  Georgia Institute of Technology \\
  Atlanta, GA 30332, USA \\
  \texttt{yni64@gatech.edu} \\
  \And
  Xiaoming Huo \\
  H.~Milton Stewart School of\\Industrial and Systems Engineering \\
  Georgia Institute of Technology \\
  Atlanta, GA 30332, USA \\
  \texttt{huo@gatech.edu} \\
}

\begin{document}
\maketitle

\begin{abstract}
The Maximum Mean Discrepancy (MMD) is a cornerstone statistic for nonparametric two-sample testing, but its test power is dictated entirely by the chosen kernel. Because any fixed kernel inherently fails to distinguish certain distributions, the kernel must be dynamically optimized. However, data-driven optimization violates the foundational i.i.d. assumption, forcing a strict trade-off in existing frameworks. Ratio criteria ignore this dependence, inducing overfitting and variance collapse on rich kernel classes. Conversely, aggregation methods bypass the dependence using finite grids, but this strategy cannot scale to continuous search spaces like deep kernels.

To break this dichotomy, we establish data-driven kernel selection as a model selection problem. We propose Complexity-Penalized MMD (CP-MMD), a criterion derived by applying the two-sample uniform concentration inequality of \citet{ni2024uci} to the post-optimization MMD problem. The resulting penalty bounds the empirical MMD by the complexity of the kernel search space, mathematically absorbing the cost of optimization, so that CP-MMD enables direct, grid-free maximization over continuous parametric classes, including scalar bandwidths, polynomial feature bandwidths, and deep network parameters. By formally accounting for optimization complexity, we prove that CP-MMD maximizes true test power while ensuring unconditional Type-I validity. Consequently, CP-MMD enables grid-free kernel selection across linear, polynomial-feature, and deep regimes, matching or exceeding state-of-the-art test power.
\end{abstract}

\section{Introduction}
\label{sec:intro}

Two-sample testing, which asks whether two datasets $\mathbf{X} \sim P$ and $\mathbf{Y} \sim Q$ are drawn from the same distribution, is a fundamental problem in modern statistics and machine learning. It arises in \emph{dataset shift detection}, where $P$ denotes training data and $Q$ denotes deployment data; in \emph{generative model evaluation}, where $P$ is the data distribution and $Q$ is the model distribution~\citep{sutherland2017generative}; and in \emph{industrial process monitoring}, where $P$ corresponds to normal operation and $Q$ to a potential fault~\citep{downs1993plant}. 

The Maximum Mean Discrepancy \citep[MMD;][]{gretton2012kernel} is a standard kernel-based statistic for two-sample testing. Its effectiveness depends critically on the choice of kernel: whether the test is powerful depends on how well the kernel aligns with the specific distributional difference between $P$ and $Q$~\citep{ramdas2015adaptivity,sutherland2017generative}. No single fixed kernel is therefore uniformly well suited to all alternatives encountered in practice, making data-driven kernel selection essential. 

However, this data-driven selection introduces statistical dependence, the primary technical obstacle in MMD-based testing. The asymptotic distribution theory and related power guarantees of the empirical MMD rely strictly on the assumption that the kernel is chosen independently of the test samples. Dynamically optimizing the kernel using the observed data violates this assumption, invalidating classical fixed-kernel asymptotics. Consequently, kernel selection cannot be treated merely as an algorithmic preprocessing step; its statistical impact on the test distribution must be explicitly resolved. Existing frameworks attempt to address this dependence through two primary paradigms, both of which force a strict statistical trade-off.

\subsection{Existing approaches and their limits}
\label{sec:existing}
First, \textbf{ratio criteria} \citep{sutherland2017generative,liu2020learning} maximize the empirical MMD normalized by its estimated standard deviation, an objective derived from asymptotic theory under the assumption that the kernel is selected independently of the input samples. Applying this data-independent proxy to rich, data-driven kernel classes introduces two critical failures. Because the optimization breaks the independence assumption, the framework severely overfits; the optimizer artificially inflates the ratio on the observed sample without improving actual test power. Furthermore, the unconstrained objective is vulnerable to \emph{variance collapse} (documented empirically in App.~\ref{app:variance-collapse}, Table~\ref{tab:variance-collapse}): the optimizer frequently drives both the empirical MMD and its variance to zero, producing a spuriously large ratio that selects degenerate kernels.
Second, rather than addressing the intractable distribution of the empirical MMD equipped with a data-driven kernel, \textbf{aggregation methods} \citep{schrab2023mmd,biggs2023mmdfuse} combine the statistics from a finite, pre-specified grid of candidate kernels. This allows the test to leverage the best-performing kernel for a specific sample while using a penalty to correct for the multiple comparisons. While tractable for low-dimensional parameters such as scalar bandwidths, this strategy cannot scale to polynomial features or learned deep representations, as finite grids cannot effectively cover continuous search spaces.

\subsection{Our approach: a single complexity-penalized criterion}
\label{sec:intro-our}

\paragraph{Composite kernel unification.} We demonstrate that data-driven kernel selection paradigms across linear, polynomial, and deep regimes can be mathematically unified under a single composite-kernel formulation: $k_h(x, x') := k(h(x), h(x'))$, where $k$ is a bounded-Lipschitz base kernel and $h \in \mathcal H$ is a learnable feature map. As shown in Table~\ref{tab:coverage} (\S\ref{sec:criterion-setup}), this formulation covers bandwidth selection (linear, $h: x \mapsto x/\sigma$), polynomial features (Laplacian on degree-$p$ features), and deep representations ($h = h_\theta$). Liu's $\epsilon$-mixture kernel structure lies outside Definition~\ref{def:composite-mmd}'s $k_h$ form (the parameter $\epsilon$ varies inside the kernel rather than inside the feature map). We show empirically in Appendix~\ref{app:exp-deep-ablation} that the criterion advantage of $J_{\mathrm{CP}}$ over Liu's ratio criterion persists on this out-of-scope kernel class via a mixture-aware penalty adjustment.

\paragraph{Kernel selection as model selection.}
Treating kernel selection as model selection in $\mathcal H$ lets us import statistical-learning tools designed for the post-optimization distribution problem. Consider random vectors $X \sim P$ and $Y \sim Q$, with corresponding sample matrices $\mathbf{X} \sim P^m$ and $\mathbf{Y} \sim Q^n$. Under the composite framework, the kernel selection problem reduces to selecting an optimal feature map $h$ applied directly to the input random vectors $X$ and $Y$. Let $\widehat\gamma_{k,u}^2(h)$ denote the empirical squared MMD between the transformed samples $h(\mathbf{X})$ and $h(\mathbf{Y})$, and let $\gamma_k^2(h)$ be its population counterpart. This equivalence allows us to formally resolve the post-optimization dependence by applying a uniform concentration inequality (UCI) for two-sample statistics. Specifically, \citet{ni2024uci} establish a UCI that extends \citet{maurer2019uniform}'s nonlinear-functional concentration framework to the two-sample setting and applies to the unbiased MMD estimator $\widehat\gamma_{k,u}^2(h)$, bounding the generalization gap $\sup_{h \in \mathcal H}|\widehat\gamma_{k,u}^2(h) - \gamma_k^2(h)|$ at a rate determined by the Gaussian complexity $G(\mathcal H)$. Because this bound holds uniformly over the entire search space, it remains strictly valid for the data-driven choice $\widehat h$. Applying it to the post-selection MMD problem establishes a finite-sample lower bound on the post-selection test power and gives the CP-MMD criterion:
\begin{equation}
J_{\mathrm{CP}}(h) := \widehat\gamma_{k,u}^2(h) - \widehat{C}_1 \cdot \widetilde G(h),
\end{equation}
where $\widetilde G(h)$ is a computable spectral-norm bound on the complexity along the optimization trajectory~\citep{bartlett2017spectrally}. The penalty coefficient $\widehat C_1$ is calibrated via training-set null permutations (App.~\ref{app:calibration}); the held-out half is reserved for the deployed permutation test, which provides exact finite-sample Type-I control by classical sample-split exchangeability.

\paragraph{Grid-free continuous optimization.}
CP-MMD eliminates the need for discrete candidate-grid constructions by incorporating the statistical cost of continuous optimization directly into the complexity penalty $\widetilde G(h)$. State-of-the-art aggregation methods such as MMDAgg and MMD-FUSE rely on a pre-specified finite set of $B$ candidate kernels and pay a $\sqrt{\ln B/N}$ Bonferroni-style multiplicity penalty; this discrete union bound is undefined in the continuous limit ($\alpha/B \to 0$), so aggregation cannot be applied to continuous parametric classes such as $\mathcal H_{\mathrm{deep}}$. As a sanity check on rate-tightness, our UCI specialized to a finite grid of $B$ kernels recovers the same $\Theta(\sqrt{\ln B/N})$ rate that \citet{schrab2023mmd}, Thm.~9 establish for MMDAgg in this regime (Prop.~\ref{prop:mmdagg-coverage}, App.~\ref{app:proof-mmdagg}), confirming our continuous-class bound matches the discrete-grid rate where the tight comparison applies.

\paragraph{Beyond formal coverage.}
Even on Liu's $\epsilon$-mixture kernel structure, which lies outside our formal coverage, applying $J_{\mathrm{CP}}$ with a mixture-aware penalty beats Liu's ratio criterion (Appendix~\ref{app:exp-deep-ablation}, Table~\ref{tab:ablation-factorial}), showing the criterion design generalizes beyond Definition~\ref{def:composite-mmd}'s $k_h$ form.

\paragraph{Contributions.}
\begin{enumerate}[leftmargin=*]
\item \emph{Unified grid-free criterion.} We introduce $J_{\mathrm{CP}}$, a single penalized objective that replaces discrete Bonferroni-corrected grids with a continuous optimization framework. This criterion unifies diverse kernel parametrizations, ranging from scalar scales to neural feature maps, under a single model class.
\item \emph{Theoretical power guarantees.} We derive a finite-sample test power lower bound for the kernel selected via our proposed criterion (Cor.~\ref{cor:power-cpmmd}) by applying the two-sample uniform concentration inequality of \citet{ni2024uci} (restated here as Theorem~\ref{thm:uci}) to the post-selection MMD problem. To our knowledge, this represents the first theoretical analysis of post-selection test power for optimized empirical MMD.
\item \emph{Architectural efficiency.} CP-MMD retains only the single optimized kernel at test time, resulting in a per-test deployment cost of $O(N_{\mathrm{perm}} \cdot N^2)$ that is independent of candidate class size. This bypasses the $B$-dependency inherent in aggregation methods: MMDAgg scales as $O(B \cdot N_{\mathrm{perm}} \cdot N^2)$ because each bootstrap iteration re-evaluates all $B$ candidates, while MMD-FUSE scales as $O((B + N_{\mathrm{perm}}) \cdot N^2)$ via shared distance matrices. Unlike these baselines, CP-MMD maintains a constant deployment cost as the search space expands.
\end{enumerate}

\paragraph{Empirical validation.}
A single criterion beats every regime-specific specialist (\S\ref{sec:exp-three-regimes}): the median heuristic on multi-scale data ($1.00$ vs.\ $0.47$ at mean shift $\Delta{=}0.30$); MMDAgg on the mildest kurtosis shift (high-dim Student-$t$, df${=}20$ degrees of freedom: $0.90$ vs.\ $0.80$); and Liu/Plain on the same MLP across high-dimensional Gaussian mean-shift cells, where CP-MMD dominates Liu (variance collapse) and Plain (overfitting at small $n$).


\section{Preliminaries}
\label{sec:prelim}
\label{sec:prelim-mmd}

\paragraph{Two-sample test.} Given independent samples $\mathbf{X} = (X_1, \ldots, X_m) \overset{\text{iid}}{\sim} P$ and $\mathbf{Y} = (Y_1, \ldots, Y_n) \overset{\text{iid}}{\sim} Q$ in $\mathbb{R}^d$.
A two-sample test aims to determine whether the difference between the population distributions $P$ and $Q$ is statistically significant.
More specifically, the null and alternative hypotheses are defined as follows:
\[
    H_0: P=Q, \quad \text{and} \quad H_1: P\neq Q.
\]
Denote the decision function as $\phi(\mathbf X, \mathbf Y) \in \{0, 1\}$, mapping the observed data to a binary reject/accept decision for the null hypothesis $H_0$ (with $\phi = 1$ denoting rejection).
The \textbf{significance level} of the test, $\alpha$, is the threshold probability used in hypothesis testing to determine whether to reject the null hypothesis $H_0$, that is,
\[
    \alpha := \sup_{H_0} \mathbb{E}[\phi(\mathbf{X}, \mathbf{Y})].
\]
The \textbf{test power} at a specific alternative $(P, Q)$ with $P \ne Q$ is
\[
    \beta(P, Q) := \mathbb{E}_{\mathbf X \sim P^m,\,\mathbf Y \sim Q^n}[\phi(\mathbf{X}, \mathbf{Y})],
\]
i.e., the probability of rejecting $H_0$ under that alternative. We treat power point-wise per alternative; a global infimum over $H_1$ is degenerate because $H_1$ contains alternatives arbitrarily close to the null, where any level-$\alpha$ test has rejection probability close to $\alpha$.
Decreasing the significance level ($\alpha$) inherently reduces statistical power $\beta(P, Q)$ for a fixed sample size~\citep{casella2002statistical}. Due to this trade-off, the Neyman-Pearson lemma~\citep{neyman1933problem} establishes the standard optimization procedure: select a decision function $\phi$ that maximizes $\beta(P, Q)$ subject to a fixed $\alpha$.

\paragraph{Maximum Mean Discrepancy.} Let $\mathcal{U} \subseteq \mathbb{R}^d$, and let $k: \mathcal{U} \times \mathcal{U} \to \mathbb{R}$ be a reproducing kernel with an induced reproducing kernel Hilbert space (RKHS) $\mathcal{F}_k$. For probability measures $P$ and $Q$ defined on $\mathcal{U}$, the maximum mean discrepancy (MMD) between $P$ and $Q$ is~\citep{gretton2012kernel}:
\begin{equation}
\label{eq:mmd-def}
\gamma_k^2(P, Q) = \mathbb{E}[k(X, X')] + \mathbb{E}[k(Y, Y')] - 2\mathbb{E}[k(X, Y)],
\end{equation}
where $X, X' \overset{\text{iid}}{\sim} P$ and $Y, Y' \overset{\text{iid}}{\sim} Q$. 
Given independent samples $\mathbf{X} = (X_1, \ldots, X_m) \overset{\text{iid}}{\sim} P$ and $\mathbf{Y} = (Y_1, \ldots, Y_n) \overset{\text{iid}}{\sim} Q$, an unbiased estimator for $\gamma_k^2(P, Q)$ is~\citep{gretton2012kernel}:
\begin{equation}
\label{eq:mmd-ustat}
\widehat\gamma_{k,u}^2(\mathbf{X}, \mathbf{Y}) = \frac{1}{m(m-1)}\sum_{i \neq j} k(X_i, X_j) + \frac{1}{n(n-1)}\sum_{i \neq j} k(Y_i, Y_j) - \frac{2}{mn}\sum_{i, j} k(X_i, Y_j).
\end{equation}
If $k$ is a characteristic kernel~\citep{sriperumbudur2010hilbert}, $\gamma_k^2(P, Q) = 0$ if and only if $P = Q$. Consequently, $\widehat\gamma_{k,u}^2$ serves as a test statistic for the two-sample problem~\citep{gretton2012kernel}. Let $c_\alpha$ denote the $(1-\alpha)$-quantile of $\widehat\gamma_{k,u}^2$ under the null hypothesis $H_0: P = Q$. A level-$\alpha$ MMD test rejects $H_0$ when $\widehat\gamma_{k,u}^2 \geq c_\alpha$, corresponding to the decision function:
\[
    \phi_{\widehat\gamma_{k,u}^2}(\mathbf{X}, \mathbf{Y}) = \mathbb{I}(\widehat\gamma_{k,u}^2 \geq c_\alpha).
\]

\paragraph{Necessity of kernel selection.} While constructing a decision function $\phi$ that bounds the significance level at $\alpha$ is straightforward via the aforementioned empirical thresholding, the statistical power of the resulting test is highly sensitive to the choice of the kernel $k$~\citep{gretton2012optimal}. 
Specifically, no single fixed kernel achieves uniform test power against multi-scale alternatives~\citep{ramdas2015adaptivity, sutherland2017generative, schrab2023mmd}. 
For any predetermined kernel bandwidth, there exist distinct alternative distributions ($P \neq Q$) for which the population MMD is arbitrarily small, causing the test's power to degrade toward the significance level $\alpha$~\citep{schrab2023mmd}, formally:
\begin{lemma}[Asymptotic triviality of fixed-bandwidth MMD tests]
\label{lem:kernel-selection-necessity}
For any fixed bandwidth $\sigma > 0$, let $k_\sigma(x, x') := \exp(-\|x - x'\|^2/(2\sigma^2))$ be the Gaussian kernel defined on $\mathbb{R}^d$.
There exists a sequence of distribution pairs $(P^{(i)}, Q^{(i)})$ defined on $\mathbb{R}^d$ for $i = 1, 2, \dots$, and a constant $c_0 > 0$, such that the total variation distance $\mathrm{TV}(P^{(i)}, Q^{(i)}) \ge c_0$ for all $i$, while the population MMD converges to zero:
\[
\gamma_{k_\sigma}^2(P^{(i)}, Q^{(i)}) \;\longrightarrow\; 0 \quad\text{as} \quad i \to \infty.
\]
Consequently, for samples $\mathbf{X}^{(i)} \overset{\text{iid}}{\sim} P^{(i)}$ and $\mathbf{Y}^{(i)} \overset{\text{iid}}{\sim} Q^{(i)}$, the probability of rejecting the null hypothesis satisfies:
\[
    \limsup_{i \to \infty} \Pr_{(P^{(i)},Q^{(i)})}\bigl(\widehat\gamma_{k,u}^2(\mathbf{X}^{(i)}, \mathbf{Y}^{(i)}) \geq c_\alpha(\mathbf{X}^{(i)}, \mathbf{Y}^{(i)})\bigr) \leq \alpha,
\]
where $c_\alpha(\mathbf{X}^{(i)}, \mathbf{Y}^{(i)})$ is the $(1-\alpha)$-quantile of the empirical null distribution computed from the pooled sample.
\end{lemma}

The lemma establishes that the asymptotic test power $\beta$ is bounded by the significance level $\alpha$ on the constructed sequence: the test performs no better than random guessing on these strictly separated alternatives. The proof, taking $P^{(i)} = \mathcal N(0, i^2)$ and $Q^{(i)} = \mathcal N(0, (2i)^2)$ on $\mathbb R$, is given in Appendix~\ref{app:proof-kernel-selection-necessity}. Choosing $k$ from data is therefore the central practical problem of MMD testing.

\section{The CP-MMD Criterion}
\label{sec:criterion}

\subsection{Composite kernels: a model class for kernel selection}
\label{sec:criterion-setup}

CP-MMD operates on the class of bounded-Lipschitz base kernels (Assumption~\ref{assum:kernel}), lifted by a feature map $h$ drawn from a model class $\mathcal H$ to form the composite kernel $k_h(x, x') = k(h(x), h(x'))$ (Definition~\ref{def:composite-mmd}). This form covers the three canonical kernel-selection settings used in modern MMD testing: bandwidth selection, polynomial features, and deep features (formal proxies in \S\ref{sec:criterion-deployed}). Liu's $\epsilon$-mixture is not a special case of Definition~\ref{def:composite-mmd} ($\epsilon$ varies inside the base kernel, not inside $h$); we apply $J_{\mathrm{CP}}$ heuristically to it in Appendix~\ref{app:exp-deep-ablation}'s ablation, where it still outperforms Liu's ratio criterion.

\begin{assumption}[Kernel regularity]
\label{assum:kernel}
$k: \mathcal{U} \times \mathcal{U} \to \mathbb{R}$ is positive-semidefinite and satisfies, for some constants $\nu > 0$ and $l > 0$:
\emph{(i) Boundedness:} $0 \le k(u, u') \le \nu$ for all $u, u' \in \mathcal U$.
\emph{(ii) Lipschitz in each argument:} for all $u, \tilde u, u' \in \mathcal U$,
$|k(u, u') - k(\tilde u, u')| \le l \|u - \tilde u\|$ (and symmetrically in the second argument).
\end{assumption}

\begin{definition}[Composite kernel and MMD]
\label{def:composite-mmd}
Let $\mathcal H$ be a class of measurable maps $h:\mathbb R^d\to\mathcal U$, the kernel-selection class. For each $h \in \mathcal H$, the \emph{composite kernel} is
\begin{equation}
{k}_{{h}}(x, x') \;:=\; k\bigl(h(x),\,h(x')\bigr),
\end{equation}
and the \emph{composite MMD} at $h$ is the population MMD evaluated under $h$'s push-forward:
\begin{equation}
\gamma_k^2(h) \;:=\; \gamma_k^2(h_\# P,\, h_\# Q) \;=\; \mathbb E\,k_h(X, X') + \mathbb E\,k_h(Y, Y') - 2\,\mathbb E\,k_h(X, Y).
\end{equation}
The unbiased estimator $\widehat\gamma_{k,u}^2(h)$ is given by~\eqref{eq:mmd-ustat} with $k$ replaced by $k_h$.
\end{definition}

\noindent The degree-$p$ \emph{polynomial feature map} $\Psi_p : \mathbb R^d \to \mathbb R^{D_p}$ is defined by
\begin{equation}
\label{eq:polynomial-feature-map}
\Psi_p(x) \;:=\; \bigl(\,x^\alpha\,\bigr)_{1 \le |\alpha| \le p} \;\in\; \mathbb R^{D_p},
\qquad
D_p \;:=\; \binom{d + p}{p} - 1,
\end{equation}
where $\alpha = (\alpha_1, \ldots, \alpha_d) \in \mathbb N^d$ is a multi-index, $|\alpha| := \sum_{i=1}^d \alpha_i$, and $x^\alpha := \prod_{i=1}^d x_i^{\alpha_i}$.

\begin{table}[!t]
\centering
\setlength{\tabcolsep}{4pt}
\renewcommand{\arraystretch}{1.20}
\caption{The composite-kernel framework $\textcolor{simpoBlueText}{k}_{\textcolor{dpoRedText}{h}}(x, x') := \textcolor{simpoBlueText}{k}(\textcolor{dpoRedText}{h}(x), \textcolor{dpoRedText}{h}(x'))$ accommodates the existing methods we compare against in \S\ref{sec:experiments}.
Denote the Gaussian kernel of bandwidth $\sigma$ by $k^{\mathrm{G}}_\sigma(x, y) := \exp\bigl(-\|x{-}y\|^2 / (2\sigma^2)\bigr)$ and the Laplacian kernel of bandwidth $\sigma$ by $k^{\mathrm{L}}_\sigma(x, y) := \exp(-\|x{-}y\|/\sigma)$.
The bandwidth selection procedure proposed in \citet{gretton2012kernel} corresponds to the linear model class $\mathcal{H}$. MMDAgg/MMD-FUSE \citep{schrab2023mmd,biggs2023mmdfuse} aggregate over a finite grid of Laplacian and Gaussian kernels at a discrete set of bandwidths (column 2); the third column shows the \emph{continuous} composite-kernel form CP-MMD covers (Laplacian base kernel composed with polynomial-feature lift), which strictly extends the discrete grid by allowing the bandwidth $\sigma$ to vary continuously in $[\sigma_l, \sigma_u]$. The definition of $\Psi_p$ appears in \eqref{eq:polynomial-feature-map}. The abbreviated version of the deep kernel proposed in \citet{liu2020learning} corresponds to the model class $\mathcal{H}$ composed of MLPs (the $q = \mathbf{I}$, $\epsilon = 0$ special case). Liu's full $\epsilon$-mixture lies outside Definition~\ref{def:composite-mmd}'s $k_h$ form; the empirical comparison between the simplified and full deep kernels is in Appendix~\ref{app:exp-deep-ablation}.}
\label{tab:coverage}
\resizebox{\linewidth}{!}{%
\begin{tabular}{@{}lll@{}}
\toprule
\textbf{Reference} & \textbf{Feasible kernel set} & \textbf{$\{\textcolor{simpoBlueText}{\boldsymbol{k}}_{\textcolor{dpoRedText}{\boldsymbol{h}}} : \textcolor{dpoRedText}{\boldsymbol{h}} \in \textcolor{green1}{\boldsymbol{\mathcal H}}\}$} \\
\midrule
\rowcolor{lightgray!20}
\citet{gretton2012kernel}
  & $\bigl\{k^{\mathrm{G}}_\sigma(x, y), \sigma \in [\sigma_l, \sigma_u]\bigr\}$
  & $\{\textcolor{simpoBlueText}{k^{\mathrm{G}}_{1}}(x\textcolor{dpoRedText}{/\sigma}, y\textcolor{dpoRedText}{/\sigma}): \textcolor{green1}{\sigma \in [\sigma_l, \sigma_u]}\}$ \\
MMDAgg \citep{schrab2023mmd} & \multirow{2}{*}{$\bigl\{k^{\mathrm{G}}_{\sigma_j}(x, y),\; k^{\mathrm{L}}_{\sigma_j}(x, y)\;:\; j = 1, \ldots, B\bigr\}$} & \multirow{2}{*}{$\bigl\{\textcolor{simpoBlueText}{k^{\mathrm{L}}_{1}}(\textcolor{dpoRedText}{\Psi_p}(x)\textcolor{dpoRedText}{/\sigma}, \textcolor{dpoRedText}{\Psi_p}(y)\textcolor{dpoRedText}{/\sigma}) : \textcolor{green1}{\sigma \in [\sigma_l, \sigma_u]} \bigr\}$} \\
MMD-FUSE \citep{biggs2023mmdfuse} & & \\
\rowcolor{lightgray!20}
\citet{liu2020learning}
  & $\bigl\{k(x, y) = k^{\mathrm{G}}_\sigma\bigl(\phi_\theta(x), \phi_\theta(y)\bigr) : (\sigma, \theta) \in \Theta\bigr\}$
  & $\{\textcolor{simpoBlueText}{k^{\mathrm{G}}_1}(\textcolor{dpoRedText}{\phi_\theta}(x)\textcolor{dpoRedText}{/\sigma}, \textcolor{dpoRedText}{\phi_\theta}(y)\textcolor{dpoRedText}{/\sigma}) : \textcolor{green1}{(\sigma, \theta) \in \Theta}\}$ \\
\bottomrule
\end{tabular}%
}
\end{table}

For example, taking $k$ Gaussian and $h(x) = x/\sigma$ recovers the Gaussian kernel at bandwidth $\sigma$; the three canonical regimes (linear/polynomial/deep) are listed in \S\ref{sec:criterion-deployed} with their computable proxies. Under this formulation, kernel selection reduces to selecting $\widehat h$ from $\mathcal H$, i.e., model selection in $\mathcal H$.

\subsection{Empirical-to-population gap: the uniform concentration inequality}
\label{sec:criterion-uci}

According to Lemma~\ref{lem:kernel-selection-necessity}, a data-driven kernel is required to maintain the test power. Nevertheless, the resulting $k_{\widehat h}$ violates the independence assumption of classical MMD theory~\citep{gretton2012kernel}, and the standard CLT-based threshold no longer characterizes rejection probability. We restore a finite-sample guarantee via the \textbf{uniform concentration inequality} (UCI) on $|\widehat\gamma_{k,u}^2(h) - \gamma_k^2(h)|$ over $h \in \mathcal H$ established by \citet{ni2024uci}, restated below for self-containment.

\begin{theorem}[Two-sample UCI; restated from \citet{ni2024uci}]
\label{thm:uci}
Let $\mathbf{X} \sim P^m$ and $\mathbf{Y} \sim Q^n$ with pooled data $\mathbf{D} = \mathbf{X} \sqcup \mathbf{Y}$ as in \S\ref{sec:prelim-mmd}, set $N = m + n$, and let $\rho_* := \max\{N/m, N/n\}$ denote the sample-size imbalance ratio (so $\rho_* \ge 2$, with equality if and only if $m = n$). Under Assumption~\ref{assum:kernel}, for any $\delta \in (0, 1)$, with probability at least $1 - \delta$:
\begin{equation}
\label{eq:conc}
\sup_{h \in \mathcal{H}} \left| \widehat{\gamma}_{k,u}^2(h) - \gamma_k^2(h) \right| \leq B_N(\delta) := \underbrace{C_1\,G(\mathcal{H})}_{\text{complexity}} + \underbrace{C_2\sqrt{\ln(2/\delta)/N}}_{\text{concentration}},
\end{equation}
where $C_1 = 2\sqrt{2\pi}\,l\,\rho_*(1+\rho_*)$, $C_2 = 4\nu\rho_*$, and $G(\mathcal H)$ is the (data-averaged) Gaussian complexity of the model class $\mathcal H$:
\[
G(\mathcal H) \;:=\; \mathbb E_{\mathbf D}\,\mathbb E_{\boldsymbol\xi}\!\left[ \sup_{h \in \mathcal H} \frac{1}{N}\sum_{i=1}^{N} \bigl\langle \boldsymbol\xi_i,\, h(\mathbf D_i) \bigr\rangle \right],
\]
with $\boldsymbol\xi_1, \ldots, \boldsymbol\xi_N$ i.i.d.\ standard Gaussian vectors in the ambient Euclidean space of $\mathcal U$ (i.e., each $\boldsymbol\xi_i \in \mathbb R^{d_{\mathrm{out}}}$ where $d_{\mathrm{out}} := \dim(\mathcal U)$) and the outer expectation taken over the data distribution.
\end{theorem}

\noindent Specialized to a finite kernel collection $\mathcal K = \{k_1, \ldots, k_B\}$ via Massart's finite-class union bound, our UCI recovers the $\Theta(\sqrt{\ln B/N})$ separation rate that \citet[Thm.~9]{schrab2023mmd} establish as minimax-tight on the discrete-grid regime (Prop.~\ref{prop:mmdagg-coverage}, App.~\ref{app:proof-mmdagg}). The bound is therefore rate-tight where a sharp comparison is known.

\paragraph{From UCI to a usable certificate.} The UCI of Theorem~\ref{thm:uci} is a population-vs-empirical bound. Specializing it to the optimization trajectory $\mathcal H_T := \{h^{(0)}, \ldots, h^{(T)}\}$ and rearranging \eqref{eq:conc} gives a deployable lower bound on $\gamma_k^2(\widehat h)$ at each visited iterate: with probability $\ge 1{-}\delta$,
\begin{equation}
\label{eq:trajectory-uci}
\gamma_k^2(h) \;\ge\; \widehat\gamma_{k,u}^2(h) - C_1\,G(\mathcal H_T) - C_2\sqrt{\ln(2/\delta)/N} =: \mathcal L_\delta(h)
\quad \text{for all } h \in \mathcal H_T.
\end{equation}
That is, if $\mathcal L_\delta(\widehat h) > 0$ at the selected kernel, then $\gamma_k^2(\widehat h) > 0$ holds with probability at least $1-\delta$. The trajectory class $\mathcal H_T$ is contained in the data-independent Lipschitz-norm ball $\{h : L(h) \le L^*\}$ on which Theorem~\ref{thm:uci} applies; the empirical spectral stabilization in App.~\ref{app:proof-spectral} (Fig.~\ref{fig:collapse}(d)) keeps $L^*$ finite under $J_{\mathrm{CP}}$ ascent. Maximizing the empirical proxy of $\mathcal L_\delta$ over $\mathcal H_T$ therefore maximizes a population power lower bound.

\subsection{Calibrated complexity penalty: the deployed criterion}
\label{sec:criterion-deployed}

The certificate \eqref{eq:trajectory-uci} contains the Gaussian complexity term $G(\mathcal H_T)$, which has no closed form. We derive a closed-form upper bound $\widetilde G$ via a Lipschitz norm ball, and pair it with a calibrated constant $\widehat C_1$ that absorbs the slack from this substitution.


\begin{proposition}[Lipschitz-norm-ball bound on $G(\mathcal H_T)$]
\label{prop:proxy-bound}
Let $G(\mathcal H_T)$ be as in Theorem~\ref{thm:uci}. Suppose each $h \in \mathcal H_T$ is $L(h)$-Lipschitz on $\mathcal U$ with $h(0) = 0$, and let $L^* := \max_{t \le T} L(h^{(t)})$. Then there is an absolute constant $c_{\mathcal H} > 0$, depending only on the regime of $\mathcal H$ (linear, polynomial, deep) and the codomain dimension $\dim(\mathcal U)$, such that
\begin{equation}
\label{eq:proxy-bound}
G(\mathcal H_T) \;\le\; c_{\mathcal H}\,L^*\cdot \tfrac{\|\mathbf D\|_F}{N}.
\end{equation}
For the deep regime, $c_{\mathcal H}$ may include an additional $\sqrt{\ln N}$ factor inherited from the spectral-norm Gaussian complexity bound of \citet[Thm.~3.1]{bartlett2017spectrally}; this $\sqrt{\ln}$ is dominated by the calibrated $\widehat C_1$ in deployment.
\end{proposition}

\noindent The per-iterate proxy is
\begin{equation}
\label{eq:contraction-bound}
\widetilde G(h) := L(h)\cdot \tfrac{\|\mathbf D\|_F}{N},
\quad
L(h) = \begin{cases}
1/\sigma & \mathcal H_{\mathrm{linear}} := \{x \mapsto x/\sigma\}\\
L_{\Psi_p}/\sigma & \mathcal H_{\mathrm{polynomial}} := \{x \mapsto \Psi_p(x)/\sigma\}\\
\prod_{j=1}^L \|W_j\|_2 & \mathcal H_{\mathrm{deep}} := \{h_\theta : \theta \in \Theta\}
\end{cases}
\end{equation}
By construction $\widetilde G(h^{(t)}) \le L^*\,\|\mathbf D\|_F/N$, with equality at termination once the empirical spectral stabilization (App.~\ref{app:proof-spectral}, Fig.~\ref{fig:collapse}(d)) caps $L(h^{(t)})$. The constants accumulated in Prop.~\ref{prop:proxy-bound} are absorbed into the calibrated $\widehat C_1$.

\paragraph{Calibrating $C_1$ by null permutations.}
The Gaussian complexity $G(\mathcal H)$ in Theorem~\ref{thm:uci}'s bound is the \emph{global} complexity of the entire candidate class. The local-vs-global Rademacher complexity gap (\citealt{bartlett2005local}, Sec.~3) shows the operationally relevant complexity along the realized trajectory is much smaller than this global proxy. Calibration absorbs the gap: instead of using the $C_1 = 2\sqrt{2\pi}\,l\,\rho_*(1{+}\rho_*) \approx 30$ from Theorem~\ref{thm:uci}, we calibrate a single constant $\widehat C_1$ that pairs with $\widetilde G$ to give an operationally tight bound.

According to the expectation version of the UCI under $H_0$: $\mathbb E[\sup_h \widehat\gamma_{k,u}^2(h)] \le C_1\,\mathbb E[G(\mathcal H)]$ (App.~\ref{app:cal-invariant}), $C_1$ \emph{upper-bounds} the ratio of expected MMD to expected Gaussian complexity; we estimate the operationally tight value by the $(1{-}\alpha)$-quantile of null-permutation ratios (a high-probability sharpening of the expectation bound; see App.~\ref{app:cal-quantile}), that is,
\begin{equation}
\label{eq:c1-def}
\widehat C_1 \;:=\; \mathrm{Quantile}_{1-\alpha}\bigl\{\widehat\gamma_{k,u}^2(\widehat h_\pi)/\widetilde G(\widehat h_\pi)\bigr\}_{\pi=1}^{n_{\mathrm{cal}}},
\end{equation}
where $\widehat h_\pi := \arg\max_{h \in \mathcal H}\widehat\gamma_{k,u}^2(h)$ on the $\pi$-permuted training half is the \emph{plain-MMD maximizer}. Plain-MMD maximization replaces $J_{\mathrm{CP}}$ ascent inside calibration on two grounds. First, the $J_{\mathrm{CP}}$ objective itself contains $\widehat C_1$ through its penalty term $\widehat C_1\,\widetilde G(h)$, so running $J_{\mathrm{CP}}$ during calibration would require $\widehat C_1$ to satisfy a fixed-point equation rather than admit the closed-form quantile estimator in~\eqref{eq:c1-def}. Second, executing $n_{\mathrm{cal}}$ full $J_{\mathrm{CP}}$ optimizations per deployment would substantially increase calibration cost. Both the plain-MMD and $J_{\mathrm{CP}}$ trajectories lie in the Lipschitz-norm ball of Proposition~\ref{prop:proxy-bound}; see App.~\ref{app:cal-procedure} for the full procedure.

Empirically, $\widehat C_1 \in [0.005, 0.014]$ across our experiments, three orders of magnitude (OOM) tighter than the worst-case $C_1$. The substitution is grounded in the UCI's expectation form~\eqref{eq:expectation-uci}: the sharpest applicable constant $\widehat C_1^*$ (eq.~\eqref{eq:c1-target}) is a class invariant depending only on $(\mathcal H, P, Q, \alpha)$, upper-bounded by $C_1 \approx 30$ but typically much smaller (App.~\ref{app:cal-invariant}'s local-vs-global complexity gap). The calibration procedure of App.~\ref{app:cal-procedure} estimates $\widehat C_1^*$ via Monte Carlo on null permutations; App.~\ref{app:cal-robustness}'s seven-OOM ablation confirms the deployed test's power is insensitive to $\widehat C_1$ across $\widehat C_1 \in [10^{-4}, 10]$, and App.~\ref{app:cal-stability-sweep}'s controlled sweeps confirm the calibrated value is stable under architecture variation (factor $\le 1.52\times$) and data-shift variation (factor $\le 1.73\times$), so a single calibration carries over to repeated deployments within the same setup.

\paragraph{The deployed criterion.}
Substituting Proposition~\ref{prop:proxy-bound} into $\mathcal L_\delta(h)$ and replacing $C_1 \to \widehat C_1$ and $L^* \to L(h)$ yields
\begin{equation}
\label{eq:criterion}
J_{\mathrm{CP}}(h) \;=\; \widehat{\gamma}_{k,u}^2(h) \;-\; \widehat{C}_1 \cdot \widetilde G(h),
\end{equation}
the empirical MMD net of a calibrated, trajectory-adaptive complexity penalty.
\vspace{-1em}
\paragraph{Deployment protocol.}
To ensure the validity of the downstream two-sample test at the selected kernel, we employ the standard sample-split plus permutation protocol of \citet{liu2020learning}. Specifically, partition $\mathbf D$ uniformly at random into halves $\mathcal D_{\mathrm{tr}}, \mathcal D_{\mathrm{te}}$. With $\widehat C_1$ from~\eqref{eq:c1-def}, run $T$ gradient-ascent steps on $J_{\mathrm{CP}}$ over $\mathcal D_{\mathrm{tr}}$ and select $\widehat h := \arg\max_{0 \le t \le T} J_{\mathrm{CP}}(h^{(t)})$. Then run the level-$\alpha$ permutation test \citep[Thm.~15.2.1]{lehmann2005testing} at $k_{\widehat h}$ on $\mathcal D_{\mathrm{te}}$ with $N_{\mathrm{perm}}$ permutations. Type-I validity follows from sample-splitting: $\widehat h$ depends only on $\mathcal D_{\mathrm{tr}}$, so $\mathcal D_{\mathrm{te}}$ is exchangeable under $H_0$.
\begin{wrapfigure}[10]{r}{0.42\textwidth}
\vspace{-1.5em}
\centering
\includegraphics[width=0.42\textwidth]{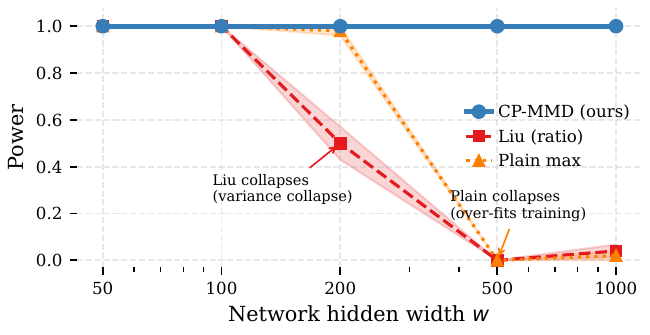}
\vspace{-1.5em}
\caption{Class-richness sweep on a fixed signal ($\pm 1$ SE): as $w$ grows, Liu collapses through variance collapse, Plain overfits, while CP-MMD holds power $1.00$.\vspace{-11em}}
\label{fig:srm}
\end{wrapfigure}

\section{Experiments}
\label{sec:experiments}

\paragraph{Common setup.}
Unless otherwise stated, all tests use significance level $\alpha=0.05$ with a permutation null based on $N_{\mathrm{perm}}=200$ permutations. Deep-kernel experiments use the same MLP architecture across methods: $d \to 200 \to 200 \to 10$ with LeakyReLU activations, $100$ Adam steps (learning rate $0.005$, gradient-clip norm $5.0$). The penalty $\widehat C_1$ is calibrated by $n_{\mathrm{cal}}=10$ null permutations on the training half (App.~\ref{app:cal-procedure}).

\subsection{Class-richness sweep: CP-MMD vs.\ Liu vs.\ plain maximization across MLP widths}
\label{sec:exp-srm}

On $P = \mathcal N(0, I_{50})$ vs.\ $Q = \mathcal N(\Delta \mathbf 1, I_{50})$ at $\Delta = 0.5$ and $n = 100$ samples per class, we sweep MLP width $w \in \{50, 100, 200, 500, 1000\}$ and compare CP-MMD ($\widehat C_1$ calibrated at $w = 200$, reused across all widths), \citet{liu2020learning}'s deep-feature kernel $\kappa(\phi_\theta(x), \phi_\theta(y))$ (the published $\epsilon$-mixture without the raw-input envelope $q$, denoted ``Liu'' below), and plain maximization ($\max_{h \in \mathcal H}\widehat{\gamma}_k^2(h)$) on the same architecture, $50$ reps per width. As $w$ grows, Liu collapses through the variance-collapse phenomenon documented in App.~\ref{app:variance-collapse} and plain max overfits, while CP-MMD holds power $1.00$ throughout (Fig.~\ref{fig:srm}). Existing methods built on the data-independent kernel assumption either collapse (Liu) or overfit (Plain) as the class grows; the $\widehat C_1\,\widetilde G(h)$ penalty in $J_{\mathrm{CP}}$ closes this post-optimization gap.

\subsection{Three-regime head-to-head: linear, polynomial, deep}
\label{sec:exp-three-regimes}

Each composite kernel family in~\eqref{eq:contraction-bound} is paired with its most natural baseline: $\mathcal H_{\mathrm{linear}}$ with the median heuristic of \citet{gretton2012kernel}, $\mathcal H_{\mathrm{polynomial}}$ with MMDAgg in its standard Laplacian-plus-Gaussian configuration, and $\mathcal H_{\mathrm{deep}}$ with \citet{liu2020learning}'s deep-feature kernel $\kappa(\phi_\theta(x), \phi_\theta(y))$ (denoted ``Liu''; published $\epsilon$-mixture without the raw-input envelope $q$). We also include plain maximization as a no-penalty baseline. The polynomial-versus-MMDAgg pairing targets the higher-moment regime via different (non-equivalent) mechanisms: bandwidth aggregation vs.\ polynomial-feature lift.

\begin{figure}[htbp]
\centering
\includegraphics[width=0.88\textwidth]{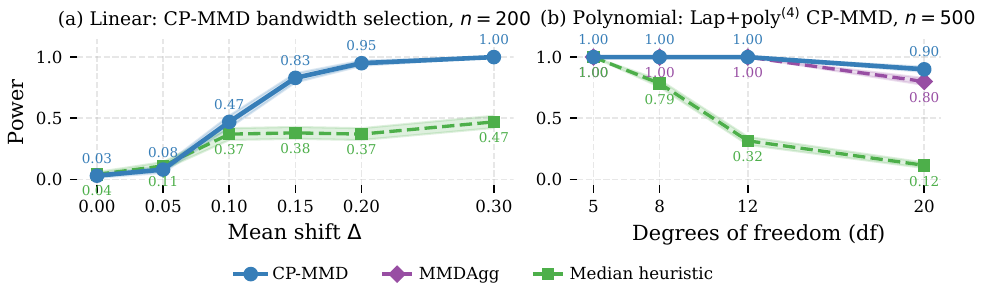}
\caption{Three-regime head-to-head (dot labels = exact power, shaded bands = $\pm 1$ binomial standard error). \textbf{(a)} Linear, multi-scale 2D mixture (axis $\Delta$ = mean shift between modes): CP-MMD bandwidth selection reaches $1.00$ at $\Delta=0.30$, Median plateaus at $0.47$. \textbf{(b)} Polynomial, kurtosis shift: the polynomial composite kernel $k^{\mathrm{L}}_1(\Psi_4(\cdot)/\sigma, \Psi_4(\cdot)/\sigma)$ matches MMDAgg's $1.00$ at df${\le}12$ and beats it by $+10$ percentage points at df${=}20$.}
\label{fig:pillar-power}
\end{figure}

\paragraph{$\mathcal H_{\mathrm{linear}}$: CP-MMD versus the median heuristic.}
On a multi-scale 2D Gaussian mixture, $P = \tfrac12 \mathcal N(0, 0.01\,I_2) + \tfrac12 \mathcal N((3, 0)^\top, 0.01\,I_2)$ and $Q = \tfrac12 \mathcal N((\Delta, 0)^\top, 0.01\,I_2) + \tfrac12 \mathcal N((3{+}\Delta, 0)^\top, 0.01\,I_2)$. We sweep $\Delta \in \{0, 0.05, 0.10, 0.15, 0.20, 0.30\}$ at $n = 200$ samples per class, $100$ reps. The median heuristic fixes the bandwidth at the median pairwise distance of the concatenated dataset $\mathcal{D}_{\mathrm{tr}}=\mathbf{X}^{\operatorname{tr}} \sqcup \mathbf{Y}^{\operatorname{tr}}$.
CP-MMD substantially outperforms the median heuristic across the shift range and reaches power $1.00$ at $\Delta = 0.30$, where the median heuristic remains at $0.47$ (Fig.~\ref{fig:pillar-power}(a)).

\paragraph{$\mathcal H_{\mathrm{polynomial}}$: CP-MMD versus MMDAgg.}
For a kurtosis-shift task in $d{=}10$ dimensions, $P = \mathcal N(0, I_{10})$ and $Q = \sqrt{(\mathrm{df}{-}2)/\mathrm{df}}\, Z$ where $Z \sim t_{\mathrm{df}}(0, I_{10})$ is multivariate Student-$t$, scaled so $\mathrm{Cov}(Q) = I_{10}$ matches $P$. We sweep $\mathrm{df} \in \{5, 8, 12, 20\}$ at $n = 500$ samples per class, $200$ reps. CP-MMD uses the polynomial composite kernel $k^{\mathrm{L}}_1(\Psi_4(\cdot)/\sigma, \Psi_4(\cdot)/\sigma)$; MMDAgg uses its standard configuration with grid sizes $B \in \{3, 5, 10, 20\}$.
CP-MMD matches MMDAgg on stronger kurtosis shifts and improves on the mildest case, reaching $0.90$ at $\mathrm{df} = 20$ versus $0.80$ (Fig.~\ref{fig:pillar-power}(b)).

\paragraph{$\mathcal H_{\mathrm{deep}}$: CP-MMD versus Liu and plain maximization.}
On the high-dimensional Gaussian mean-shift family $P = \mathcal N(0, I_d)$ vs.\ $Q = \mathcal N(\Delta \mathbf 1, I_d)$ at $\Delta = 0.5$, we evaluate five $(d, n)$ cells with $200$ reps per cell, comparing CP-MMD, Liu, and plain max on the same MLP under sample-split (App.~\ref{app:proof-power-cpmmd}). MMDAgg and MMD-FUSE require a finite kernel grid and do not apply to $\mathcal H_{\mathrm{deep}}$. CP-MMD attains power $1.00$ in four of the five cells, dropping only to $0.83$ at the mildest alternative; Plain approximately matches CP-MMD at the two $n{=}200$ cells (reaching $0.81$ and $1.00$) but degrades to $0.14$--$0.54$ at the three smaller-$n$ cells, while Liu stays strictly below CP-MMD in every cell (Table~\ref{tab:deep-summary}).

\begin{table}[htbp]
\centering\small
\setlength{\tabcolsep}{4pt}
\setlength{\abovecaptionskip}{6pt}
\begin{tabular}{lccccc}
\toprule
Criterion & \shortstack{$d{=}2$\\$n{=}200$} & \shortstack{$d{=}20$\\$n{=}200$} & \shortstack{$d{=}50$\\$n{=}100$} & \shortstack{$d{=}100$\\$n{=}100$} & \shortstack{$d{=}20$\\$n{=}50$} \\
\midrule
\textbf{CP-MMD} & \textbf{.83}$\pm$.03 & \textbf{1.00} & \textbf{1.00} & \textbf{1.00} & \textbf{1.00} \\
Liu & .30$\pm$.03 & .60$\pm$.04 & .46$\pm$.04 & .74$\pm$.03 & .16$\pm$.03 \\
Plain & .81$\pm$.03 & 1.00 & .54$\pm$.04 & .35$\pm$.03 & .14$\pm$.02 \\
\bottomrule
\end{tabular}
\caption{Deep-kernel power at $\Delta=0.5$ (HDGM, MLP $d \to 200 \to 200 \to 10$, sample-split, 200 reps). Type-I at $\alpha=0.05$ within $\pm 2$ binomial SE of nominal across all cells.}
\label{tab:deep-summary}
\end{table}

\subsection{Real data: CP-MMD across kernel parameterizations on Higgs}
\label{sec:exp-higgs}

On the UCI Higgs boson dataset \citep{baldi2014searching}, we take $P$ as background events and $Q$ as signal events with $d = 28$ kinematic features. At sample sizes $n \in \{200, 500\}$ per class and $50$ reps per cell, we evaluate four CP-MMD ans\"atze (linear, polynomial degree-$\{2, 3\}$, and a deep MLP) against the median heuristic, MMDAgg, and (for the deep ansatz) Liu and Plain on the same architecture. CP-MMD is competitive across all four ans\"atze, with the deep version substantially outperforming Liu and Plain on the same architecture (Fig.~\ref{fig:real-data}(a)). Two further benchmarks (Appendix~\ref{app:exp-real}): CIFAR-10/10.1 saturates at $n{=}200$ for all methods; on the Tennessee Eastman Process at low $n{=}200$, MMDAgg's bandwidth aggregation outperforms learned-kernel approaches.

\subsection{Deployment per-test cost: CP-MMD is $B$-free}
\label{sec:exp-runtime}

Matching the Higgs panel at $d = 28$, $n = 200$, we measure per-test wall-clock time over $30$ independent test pairs at each grid size $B \in \{5, 10, 20, 50\}$, using a JAX JIT-compiled permutation null with $N_{\mathrm{perm}} = 200$. CP-MMD evaluates one learned kernel; MMDAgg and MMD-FUSE evaluate $B$ candidates with Bonferroni correction. MMDAgg runtime grows linearly with $B$ while CP-MMD remains $B$-independent; at $B = 10$, CP-MMD-polynomial$_3$ ($399$\,ms) is $1.6\times$ faster than MMDAgg ($637$\,ms) (Fig.~\ref{fig:real-data}(b)). The trajectory-adaptive penalty and per-test cost asymptotics are formalized in Cor.~\ref{cor:adaptive-tax}.

\begin{figure}[t]
\centering
\includegraphics[width=0.92\linewidth,keepaspectratio]{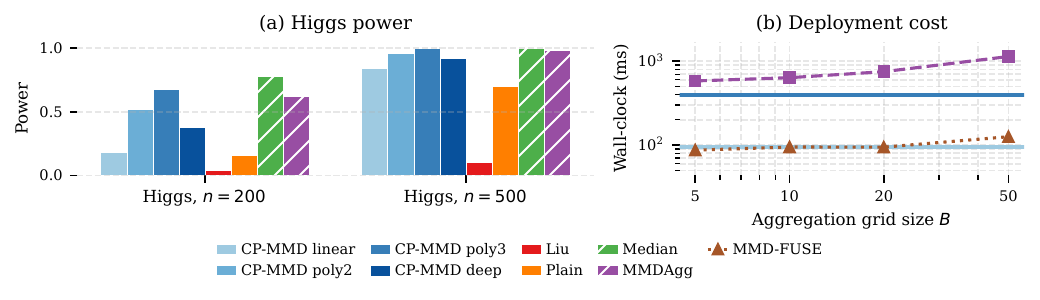}
\caption{\textbf{(a) Higgs power} ($d{=}28$, $n\in\{200,500\}$): CP-MMD-polynomial$_3$ matches Median ($1.00$) and beats MMDAgg ($0.98$) at $n{=}500$; CP-MMD-deep dominates Liu/Plain on the same MLP ($0.92$ vs.\ $0.10/0.70$). \textbf{(b) Deployment per-test cost} ($d{=}28$, $n{=}200$; log--log axes): CP-MMD references are $B$-free (flat); MMDAgg grows linearly; MMD-FUSE near-constant. At $B{=}10$, CP-MMD-polynomial$_3$ ($399$\,ms) is $1.6\times$ faster than MMDAgg ($637$\,ms).}
\label{fig:real-data}
\end{figure}

\section{Discussion and conclusion}
\label{sec:discussion}

We introduced CP-MMD, a complexity-penalized criterion that treats data-driven kernel selection as continuous model selection across linear, polynomial-feature, and deep regimes, with $B$-free deployment cost. Across four panels, a single calibrated $\widehat C_1$ keeps the same criterion competitive with method-specific alternatives in every regime: linear CP-MMD surpasses the Median heuristic on HDGM, polynomial CP-MMD matches Median on Higgs at $n{=}500$ and outperforms MMDAgg, and deep CP-MMD holds power $1.00$ across MLP widths where Liu's ratio criterion collapses and plain maximization overfits. Even on Liu's out-of-scope mixture kernel, Table~\ref{tab:ablation-factorial} shows that the CP-MMD penalty outperforms Liu's ratio criterion. Limitations are discussed in App.~\ref{app:limitations}.

\vspace{-1em}
\paragraph{Future work.} Extensions to other kernel-based statistics (HSIC~\citep{gretton2007kernel}, kernel Stein discrepancy~\citep{liu2016kernelized,chwialkowski2016kernel}), tighter regime-specific proxies, and adaptive penalization are natural directions; a fuller rate-theoretic and calibration-analytic treatment is the focus of an extended journal version, in preparation.

\bibliographystyle{plainnat}
\bibliography{references}

\newpage
\section*{NeurIPS Paper Checklist}

\begin{enumerate}

\item \textbf{Claims}
    \item[] Answer: \answerYes{}
    \item[] Justification: \S\ref{sec:intro-our} states two contributions: a unified grid-free criterion $J_{\mathrm{CP}}$ (\S\ref{sec:criterion-setup}--\S\ref{sec:criterion-deployed}) and a finite-sample test-power lower bound for the post-selection kernel (Theorem~\ref{thm:uci}, Cor.~\ref{cor:power-cpmmd}). One additional theoretical result supports these contributions: rate-tightness on finite kernel grids (Proposition~\ref{prop:mmdagg-coverage}). Empirical validation appears in \S\ref{sec:experiments}.

\item \textbf{Limitations}
    \item[] Answer: \answerYes{}
    \item[] Justification: App.~\ref{app:limitations} discusses the two principal limitations in detail, with a forward pointer in \S\ref{sec:discussion}. \emph{Calibration scope of $\widehat C_1$:} the high-quantile permutation calibration is reliable at moderate MLP width ($w \in [50, 200]$) and degenerates at wider widths, with empirical spectral stabilization (App.~\ref{app:proof-spectral}, Fig.~\ref{fig:collapse}(d)) still bounding training; calibration should be performed at a representative width and reused across architectures of comparable scale. \emph{Looseness of the complexity proxy $\widetilde G$:} the worst-case $C_1\approx 30$ and calibrated $\widehat{C}_1\approx 0.008$ differ by three orders of magnitude; the surrogate is correct in shape but conservative, with the gap absorbed empirically rather than closed in theory.

\item \textbf{Theory assumptions and proofs}
    \item[] Answer: \answerYes{}
    \item[] Justification: All formal results state their assumptions (Assumption~\ref{assum:kernel}). Full proofs appear in the supplementary material: Theorem~\ref{thm:uci} (UCI) in App.~\ref{app:proof-uci}, Proposition~\ref{prop:proxy-bound} (Lipschitz proxy) in App.~\ref{app:proxy-bounds-prop}, Corollary~\ref{cor:power-cpmmd} (power consistency) in App.~\ref{app:proof-power-cpmmd}, and Proposition~\ref{prop:mmdagg-coverage} (rate-tightness on finite kernel grids) in App.~\ref{app:proof-mmdagg}. The variance-collapse phenomenon of Liu's ratio criterion under $H_0$ (App.~\ref{app:variance-collapse}) and the spectral-norm stabilization under $J_{\mathrm{CP}}$ ascent (App.~\ref{app:proof-spectral}) are documented as empirical observations rather than formal theorems. The Type-I validity of the held-out permutation test is classical \citep[Theorem~15.2.1]{lehmann2005testing} and used without restating.

\item \textbf{Experimental result reproducibility}
    \item[] Answer: \answerYes{}
    \item[] Justification: \S\ref{sec:experiments} specifies all architectures, optimizers, learning rates, sample sizes, and repetition counts. The supplementary material provides (a) the producing scripts for every table and figure (\texttt{simulations/run\_*.py}), (b) reference CSVs under \texttt{results/} containing the exact numbers reported in Tables~\ref{tab:deep-summary}--\ref{tab:ablation-factorial} and Figs.~\ref{fig:srm}--\ref{fig:c1-ablation}, and (c) a \texttt{README.md} mapping each paper artifact to its producing script with expected runtime. Tables~\ref{tab:deep-summary}, \ref{tab:variance-collapse}, and \ref{tab:ablation-factorial} cite the relevant producing script directly in their captions.

\item \textbf{Open access to data and code}
    \item[] Answer: \answerYes{}
    \item[] Justification: Source code (16 experiment drivers, 3 core modules, the figure-rendering script) and reference CSVs are provided in the supplementary archive \texttt{mmd-kernel-selection-code.zip}, with a \texttt{README.md} mapping each paper artifact to its producing script and CSV. Real datasets are publicly available and downloaded by the scripts on first run: CIFAR-10/10.1 (\texttt{torchvision} + \texttt{modestyachts/CIFAR-10.1}), UCI Higgs (\url{https://archive.ics.uci.edu/ml/datasets/HIGGS}); the Tennessee Eastman experiment uses our own synthetic Gaussian-shift implementation (not the proprietary Downs--Vogel simulator). Public release of the code under a permissive open-source license is planned upon acceptance.

\item \textbf{Experimental setting/details}
    \item[] Answer: \answerYes{}
    \item[] Justification: \S\ref{sec:experiments} provides network architectures, optimizer settings, sample sizes, repetition counts, and significance levels for all experiments.

\item \textbf{Experiment statistical significance}
    \item[] Answer: \answerYes{}
    \item[] Justification: Standard errors are reported in Table~\ref{tab:deep-summary} (200 reps per cell, binomial SE) and for the Higgs results in \S\ref{sec:exp-higgs}; Table~\ref{tab:variance-collapse} (App.~\ref{app:variance-collapse}) reports means over 10 seeds per width; Table~\ref{tab:ablation-factorial} (App.~\ref{app:exp-deep-ablation}) uses 100 reps per cell with implicit binomial SE. The binomial SE convention $\mathrm{SE} = \sqrt{p(1-p)/n_{\mathrm{rep}}}$ is used for the shaded bands in Figs.~\ref{fig:srm} and~\ref{fig:pillar-power}. Type-I error at $\alpha=0.05$ is verified to lie within $\pm 2$ binomial SE of nominal in every cell of Table~\ref{tab:deep-summary}.

\item \textbf{Experiments compute resources}
    \item[] Answer: \answerYes{}
    \item[] Justification: \S\ref{sec:exp-higgs} reports per-test deployment cost. CPU-only suffices for the linear and polynomial pillars (\(\sim\)10--30 minutes each); the deep-kernel and class-richness experiments use a single GPU (NVIDIA A100 in our runs, \(\sim\)1--3 hours per driver); calibration sweeps run in \(\sim\)5--10 minutes on GPU; the deployment-cost experiment (Figure~\ref{fig:real-data}(b)) is GPU-accelerated via JAX JIT. Per-script runtime estimates are provided in the supplementary \texttt{README.md}.

\item \textbf{Code of ethics}
    \item[] Answer: \answerYes{}
    \item[] Justification: The research is methodological (statistical testing) with no ethical concerns.

\item \textbf{Broader impacts}
    \item[] Answer: \answerNA{}
    \item[] Justification: This is foundational statistical methodology. It improves two-sample testing, a generic statistical tool with no direct negative societal implications.

\item \textbf{Safeguards}
    \item[] Answer: \answerNA{}
    \item[] Justification: No models or datasets with misuse risk are released.

\item \textbf{Licenses for existing assets}
    \item[] Answer: \answerYes{}
    \item[] Justification: CIFAR-10 (Krizhevsky, 2009) is publicly distributed by the University of Toronto for academic research; CIFAR-10.1 (Recht et al., 2018) is released under the MIT License via the \texttt{modestyachts/CIFAR-10.1} GitHub repository; the UCI Higgs dataset~\citep{baldi2014searching} is available under CC BY 4.0 from the UCI Machine Learning Repository. The Tennessee Eastman process is referenced via our own synthetic Gaussian-shift implementation; we do not redistribute the proprietary Downs--Vogel simulator.

\item \textbf{New assets}
    \item[] Answer: \answerNA{}
    \item[] Justification: No new datasets or models are released in this paper.

\item \textbf{Crowdsourcing and research with human subjects}
    \item[] Answer: \answerNA{}
    \item[] Justification: No human subjects or crowdsourcing involved.

\item \textbf{Institutional review board (IRB) approvals or equivalent for research with human subjects}
    \item[] Answer: \answerNA{}
    \item[] Justification: No human subjects research.

\item \textbf{Declaration of LLM usage}
    \item[] Answer: \answerYes{}
    \item[] Justification: LLMs were used to assist with prose editing and proofreading of the manuscript and with experimental code (e.g., refactoring, helper functions, plotting scripts). No LLM is part of the core methodology or the core theoretical analysis.

\end{enumerate}

\newpage
\appendix
\begin{center}
{\Large\bfseries Supplementary Material}\\[0.5em]
{\large Complexity-Penalized MMD: Appendices}
\end{center}

\section{Notation}
\label{app:notation}
\begin{center}\footnotesize
\begin{tabular}{@{}l p{0.78\textwidth}@{}}
\toprule
Symbol & Meaning \\
\midrule
$\mathbf X, \mathbf Y$ & i.i.d.\ samples $(X_1,\ldots,X_m) \overset{\mathrm{iid}}{\sim} P$ and $(Y_1,\ldots,Y_n) \overset{\mathrm{iid}}{\sim} Q$ in $\mathbb R^d$ (equivalently $\mathbf X \sim P^m$, $\mathbf Y \sim Q^n$) \\
$N = m + n$ & total sample size; $\rho_* = \max\{N/m, N/n\}$ \\
$\mathbf D = \mathbf X \sqcup \mathbf Y$ & pooled data ($\mathbf D_i$ denotes the $i$-th element, see \S\ref{sec:prelim-mmd}); $\mathcal D_F = \|\mathbf D\|_F$ \\
$\mathcal D_{\mathrm{tr}}, \mathcal D_{\mathrm{te}}$ & training / test halves of $\mathbf D$ under the sample-split protocol (\S\ref{sec:criterion-deployed}) \\
$k, \nu, l$ & base kernel; boundedness $\nu$; Lipschitz $l$ (Assumption~\ref{assum:kernel}) \\
$h \in \mathcal H$ & kernel-selection map $h:\mathbb R^d \to \mathcal U$; composite kernel $k_h(x,x') = k(h(x),h(x'))$ (Definition~\ref{def:composite-mmd}) \\
$\gamma_k^2(h)$ & population MMD at $h$; $\widehat\gamma_{k,u}^2(h)$ unbiased estimator \\
$\mathcal H_T$ & trajectory class $\{h^{(0)},\ldots,h^{(T)}\}$ visited by the optimizer over $T$ ascent steps \\
$G(\mathcal H), \widetilde G(h)$ & Gaussian complexity of $\mathcal H$; per-$h$ spectral-norm Lipschitz proxy \eqref{eq:contraction-bound} \\
$C_1, \widehat{C}_1$ & worst-case UCI constant (Theorem~\ref{thm:uci}); null-permutation calibrated hyperparameter (App.~\ref{app:calibration}) \\
$J_{\mathrm{CP}}(h)$ & CP-MMD criterion $\widehat\gamma_{k,u}^2(h) - \widehat{C}_1\widetilde G(h)$, \eqref{eq:criterion} \\
$J_{\mathrm{Liu}}(h)$ & Liu's ratio criterion $\sqrt n \widehat\gamma_{k,u}^2(h)/\hat\tau_h$ (Appendix~\ref{app:variance-collapse}) \\
$\widehat h, \widehat h_\pi$ & CP-MMD selector (\S\ref{sec:criterion-deployed}); plain-MMD maximizer on $\pi$-permuted training half (\eqref{eq:c1-def}) \\
$\Pi^{(t)}, \Pi^*$ & product of layer spectral norms at step $t$; empirical stable equilibrium under $J_{\mathrm{CP}}$ ascent (App.~\ref{app:proof-spectral}) \\
$B_N(\delta)$ & UCI bound $C_1 G(\mathcal H) + C_2\sqrt{\ln(2/\delta)/N}$ (Theorem~\ref{thm:uci}) \\
$n_{\mathrm{cal}}, N_{\mathrm{perm}}$ & number of null permutations for $\widehat C_1$ calibration; number of permutations in the deployed test \\
$\Delta$ (capital) & mean-shift parameter in §\ref{sec:experiments} ($Q = \mathcal N(\Delta\cdot\mathbf 1, I_d)$); distinct from confidence $\delta$ in Theorem~\ref{thm:uci} \\
HDGM & high-dimensional Gaussian mean-shift family (\S\ref{sec:exp-three-regimes}, Table~\ref{tab:deep-summary}) \\
\bottomrule
\end{tabular}
\end{center}

\section{Proof of Lemma~\ref{lem:kernel-selection-necessity}}
\label{app:proof-kernel-selection-necessity}

\textbf{Proof.} Take $P^{(i)} = \mathcal N(0, i^2)$ and $Q^{(i)} = \mathcal N(0, (2i)^2)$ on $\mathbb R$.

\paragraph{Total variation is constant in $i$.} Both densities are obtained from $\mathcal N(0,1)$ and $\mathcal N(0,4)$ by the dilation $z \mapsto i z$. Total variation is invariant under joint rescaling of the two distributions, so
\[
\mathrm{TV}(P^{(i)}, Q^{(i)}) = \mathrm{TV}(\mathcal N(0,1), \mathcal N(0,4)) =: c_0,
\]
which is a strictly positive constant independent of $i$.

\paragraph{Population MMD vanishes.} The Gaussian-kernel inner products under Gaussian distributions admit closed forms via the identity $\mathbb E_{Z \sim \mathcal N(\mu,\Sigma)}[\exp(-\alpha Z^2)] = (1+2\alpha\Sigma)^{-1/2}\exp\bigl(-\alpha\mu^2/(1+2\alpha\Sigma)\bigr)$. With $\alpha = 1/(2\sigma^2)$ to match the kernel definition $k_\sigma(x, x') = \exp(-\|x-x'\|^2/(2\sigma^2))$:
\begin{align*}
\mathbb E\,k_\sigma(X, X') &= (1 + 2i^2/\sigma^2)^{-1/2}, \quad X, X' \overset{\mathrm{iid}}{\sim} P^{(i)} \text{ so } X-X' \sim \mathcal N(0, 2i^2),\\
\mathbb E\,k_\sigma(Y, Y') &= (1 + 8i^2/\sigma^2)^{-1/2}, \quad Y, Y' \overset{\mathrm{iid}}{\sim} Q^{(i)} \text{ so } Y-Y' \sim \mathcal N(0, 8i^2),\\
\mathbb E\,k_\sigma(X, Y) &= (1 + 5i^2/\sigma^2)^{-1/2}, \quad X \sim P^{(i)}, Y \sim Q^{(i)} \text{ so } X-Y \sim \mathcal N(0, 5i^2).
\end{align*}
All three quantities are $O(1/i)$ as $i \to \infty$, so $\gamma_{k_\sigma}^2(P^{(i)}, Q^{(i)}) \to 0$ by~\eqref{eq:mmd-def}.

\paragraph{Test power collapses.} Fix any $\sigma > 0$ and the level-$\alpha$ permutation test $\phi_{k_\sigma}$ on $m + n$ samples with $N_{\mathrm{perm}}$ permutations. Under the alternative $(P^{(i)}, Q^{(i)})$ the unbiased MMD estimator $\widehat\gamma_{k_\sigma, u}^2$ has mean $\gamma_{k_\sigma}^2(P^{(i)}, Q^{(i)}) = O(1/i) \to 0$ (just shown) and variance $\le \nu^2 / \min(m, n)$ uniformly in $i$ via boundedness $|k_\sigma| \le \nu$ and the MMD seminorm calculation of Theorem~\ref{thm:uci}. By Chebyshev's inequality, the alternative-sample MMD concentrates within $O(\nu/\sqrt{\min(m,n)})$ of zero with high probability for $i$ large enough that $\gamma_{k_\sigma}^2(P^{(i)}, Q^{(i)}) \ll \nu/\sqrt{\min(m,n)}$. By the same boundedness argument applied to the permutation distribution of $\widehat\gamma_{k_\sigma, u}^2$ under random labeling of the pooled sample (variance $\le \nu^2/\min(m,n)$), Chebyshev's upper-tail bound gives $c_\alpha = O(\nu/\sqrt{\min(m,n)})$. As $i \to \infty$, the alternative-sample mean $\gamma_{k_\sigma}^2(P^{(i)},Q^{(i)}) = O(1/i)$ shrinks to $0$ while $c_\alpha$ remains at scale $O(\nu/\sqrt{\min(m,n)})$ with matching variance bound, so the alternative-sample distribution of $\widehat\gamma_{k_\sigma, u}^2$ converges to the permutation-null distribution; the level-$\alpha$ test's rejection probability under $(P^{(i)}, Q^{(i)})$ therefore converges to its size $\alpha$ under $H_0$:
\[
\liminf_{i \to \infty}\,\Pr_{(P^{(i)}, Q^{(i)})}(\phi_{k_\sigma} = 1) \;=\; \alpha,
\]
i.e., the test's rejection probability collapses to its size, regardless of $\sigma > 0$. The collapse is a consequence of the dilation: at any fixed bandwidth $\sigma$, the dilated samples are spread over a region of scale $i$ and the Gaussian kernel matrix becomes near-zero off the diagonal, so the resulting MMD estimator carries no signal. A single fixed $\sigma$ therefore cannot sustain power along the dilation sequence; only a kernel whose bandwidth scales with $i$ (i.e., a data-driven choice) can. \qed

\section{Variance collapse of the ratio criterion: full statement and verification}
\label{app:variance-collapse}

This appendix makes precise the failure mode of the ratio criterion and gives empirical evidence; it motivates the subtractive form of $J_{\mathrm{CP}}$ in \S\ref{sec:criterion}.

The ratio criterion's plug-in $z$-statistic $J_{\mathrm{Liu}}(h) := \sqrt{n}\,\widehat\gamma_{k,u}^2(h)/\hat\tau_h$ is derived from asymptotic theory that assumes $h$ is fixed independently of the data~\citep{liu2020learning,sutherland2017generative}. When $h$ is instead optimized on the same data used to compute the statistic, this independence assumption breaks. Consequently, the optimizer can drive both the empirical MMD $\widehat\gamma_{k,u}^2(h)$ and the variance estimate $\hat\tau_h$ to zero simultaneously when the involved kernel class is rich, producing a spuriously large ratio at a degenerate kernel that carries no actual signal. We document this variance-collapse phenomenon directly: training discriminators to maximize $J_{\mathrm{Liu}}$ under $H_0$ ($P = Q$) routinely produces $\widehat\gamma_{k,u}^2(h) \approx 0$ alongside $J_{\mathrm{Liu}}(h)$ values in the $10^3$--$10^4$ range, with $\hat\tau_h$ collapsing to numerical-precision scale ($\sim 10^{-6}$).

\paragraph{Empirical observation: variance collapse under $H_0$.}
\label{obs:liu-failure}
We document the variance-collapse phenomenon directly. Setup: $P = Q = \mathcal N(0, I_{10})$ (so $H_0$ holds), $m = n = 200$, $200$ Adam steps to maximize $J_{\mathrm{Liu}}$ on a deep-MLP discriminator. A run is flagged as \emph{collapsed} if all three of $\widehat\gamma_{k,u}^2 < 0.01$, $\hat\tau_h < 0.001$, and $J_{\mathrm{Liu}} > 10$ are observed at termination. Across MLP widths $w \in \{10, 50, 200, 400\}$ (10 seeds each), collapse occurs in $20\%$--$80\%$ of runs and $J_{\mathrm{Liu}}$ reaches up to $24{,}907$ while $\widehat\gamma_{k,u}^2 \approx 0$ (Table~\ref{tab:variance-collapse}, Figure~\ref{fig:collapse}). The mechanism is that $\hat\tau_h$ collapses to numerical-precision scale ($\sim 10^{-6}$) faster than $\widehat\gamma_{k,u}^2$ shrinks, so the ratio $\widehat\gamma_{k,u}^2/\hat\tau_h$ explodes despite the kernel carrying no actual signal.

By contrast, CP-MMD's $J_{\mathrm{CP}}$ stays strictly negative in every collapsed run: whenever $\widehat\gamma_{k,u}^2 \le 0$, the subtractive form forces $J_{\mathrm{CP}} \le -\widehat C_1\widetilde G < 0$, so no degenerate kernel can be selected.

A clean asymptotic-rate statement of variance collapse for the unbiased MMD $\widehat\gamma_{k,u}^2$ paired with the plug-in $\hat\tau_h$ is delicate, because under $H_0$ the unbiased estimator $\widehat\gamma_{k,u}^2$ has the chi-squared-like limiting distribution of a degenerate U-statistic with rate $\Theta_p(1/n)$ \citep[Theorem~12 and Lemma~6]{gretton2012kernel}: under balanced sampling $m = n$, $\widehat\gamma_{k,u}^2$ admits a one-sample U-statistic representation in the paired data $z_i := (x_i, y_i)$ with kernel $h(z_i, z_j) = k(x_i, x_j) + k(y_i, y_j) - k(x_i, y_j) - k(x_j, y_i)$; this kernel is degenerate at $H_0$, yielding the $t\cdot\widehat\gamma_{k,u}^2 \to \sum_l \lambda_l[(\rho_x^{-1/2}a_l - \rho_y^{-1/2}b_l)^2 - (\rho_x\rho_y)^{-1}]$ limit of \citet[Theorem~12]{gretton2012kernel}. The $\Theta_p(1/n)$ rate cited above is a direct consequence of this distributional result. The variance estimator $\hat\tau_h$, an empirical proxy for the $H_1$-variance, also tends to zero under $H_0$, so a clean ratio-rate analysis is delicate. We therefore document the phenomenon empirically rather than as a formal proposition; the empirical evidence below is the basis for adopting the subtractive form of $J_{\mathrm{CP}}$ over a ratio.

\begin{table}[h]
\centering
\small
\begin{tabular}{rrrrrrr}
\toprule
Width & Collapse rate & Mean $J_{\mathrm{Liu}}$ & Mean $\widehat\gamma_{k,u}^2$ & Mean $\hat\tau_h$ & Mean $\widetilde G$ & Mean $J_{\mathrm{CP}}$ \\
\midrule
10  &  80\% &    706 & 0.005 & $9 \times 10^{-6}$ & 116 & $\mathbf{-0.111}$ \\
50  &  20\% & 11{,}431 & 0.016 & $1 \times 10^{-6}$ & 100 & $\mathbf{-0.084}$ \\
200 &  70\% & 5{,}578 & 0.008 & $1 \times 10^{-6}$ &  88 & $\mathbf{-0.080}$ \\
400 &  30\% &    126 & 0.000 & $8 \times 10^{-6}$ & 221 & $\mathbf{-0.221}$ \\
\bottomrule
\end{tabular}
\caption{Variance collapse under $H_0$ for $J_{\mathrm{Liu}}$ across network widths (mean of 10 seeds). The ``Collapse rate'' column counts runs satisfying all three flags ($\widehat\gamma_{k,u}^2 < 0.01$, $\hat\tau_h < 0.001$, $J_{\mathrm{Liu}} > 10$); CP-MMD's $J_{\mathrm{CP}}$ stays strictly negative on every such run. The collapse rate at $w=400$ is lower than at narrower widths because in some seeds the optimizer instead settles into a near-constant feature map driving $\widehat\gamma_{k,u}^2$ and $\hat\tau_h$ to numerical zero together, suppressing the ratio explosion that triggers the flag; $J_{\mathrm{CP}}$ remains negative regardless. Reproduced by \texttt{simulations/run\_variance\_collapse.py}.}
\label{tab:variance-collapse}
\end{table}

\begin{figure}[h]
\centering
\includegraphics[width=\textwidth]{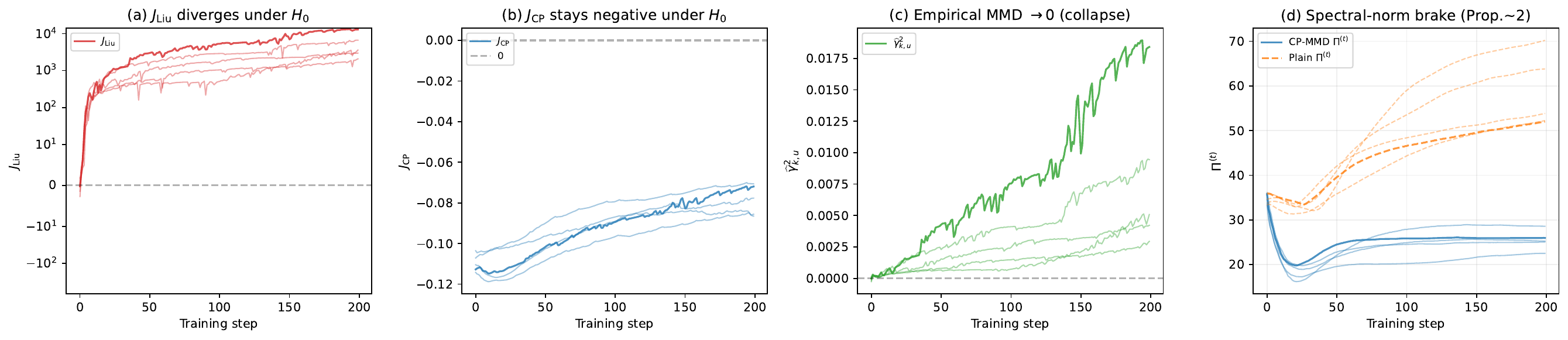}
\caption{Training dynamics under $H_0$, five independent runs. \textbf{(a)} $J_{\mathrm{Liu}}$ diverges to $10^3$--$10^4$ despite $P = Q$. \textbf{(b)} $J_{\mathrm{CP}}$ stays negative. \textbf{(c)} The empirical MMD hovers near zero (collapse). \textbf{(d)} Spectral-norm product along the trajectory, $\Pi^{(t)} := \prod_{j=1}^{L}\|W_j^{(t)}\|_2$ (the y-axis), plateaus under CP-MMD while growing unboundedly under plain maximization.}
\label{fig:collapse}
\end{figure}


\section{Proof of Theorem~\ref{thm:uci} (UCI)}
\label{app:proof-uci}

This appendix restates the two-sample UCI of \citet{ni2024uci} (\S\ref{app:proof-uci-twosample}, Lemma~\ref{lem:two-sample-uci}) and specializes it to the unbiased MMD estimator $\widehat\gamma_{k,u}^2$ (\S\ref{app:proof-uci-mmd}). The Maurer-style symmetrization-and-induction argument and the MMD-specific seminorm derivation are reproduced here for self-containment; see \citet{ni2024uci} for the general theory and its applications to other kernel-based statistics (Energy Distance, distance Covariance, and HSIC).

\subsection{Two-sample concentration framework}
\label{app:proof-uci-twosample}

\paragraph{Setup.} Let $\mathbf X = (X_1, \ldots, X_m) \in \mathcal U^m$ and $\mathbf Y = (Y_1, \ldots, Y_n) \in \mathcal U^n$ be independent samples, $\mathbf X', \mathbf Y'$ independent copies, $\mathbf D = \mathbf X \sqcup \mathbf Y$ the pooled data, and $N = m + n$. Let $f: \mathcal U^{m+n} \to \mathbb R$ be a measurable two-sample functional that takes the concatenation $(\mathbf X, \mathbf Y)$ as its input. Let $\mathcal H$ be a class of measurable maps $\mathbb R^d \to \mathcal U$ with finite covering number under $\|\cdot\|_\infty$.

\paragraph{Lipschitz seminorms of $f$.} Following \citet[Assumption~1]{maurer2019uniform}, define the first-order Lipschitz, second-order (cross-coordinate) Lipschitz, and McDiarmid bounded-difference seminorms
\begin{align*}
M_{\mathrm{Lip}}(f) &:= N\cdot\max_k \sup_{\mathbf u \in \mathcal U^N,\, y \ne y' \in \mathcal U} \frac{|D^k_{y y'} f(\mathbf u)|}{\|y - y'\|}, \\
J_{\mathrm{Lip}}(f) &:= N^2\cdot\max_{k \ne l} \sup_{\mathbf u,\, y \ne y',\, z, z'} \frac{|D^l_{z z'} D^k_{y y'} f(\mathbf u)|}{\|y - y'\|}, \\
M(f) &:= N\cdot\max_k \sup_{\mathbf u,\, y, y' \in \mathcal U} |D^k_{y y'} f(\mathbf u)|,
\end{align*}
where $D^k_{y y'} f(\mathbf u) := f(\ldots, u_{k-1}, y, u_{k+1}, \ldots) - f(\ldots, u_{k-1}, y', u_{k+1}, \ldots)$ is the $k$-th coordinate-swap operator.

\begin{lemma}[Two-sample UCI for nonlinear functionals; \citealt{ni2024uci}]
\label{lem:two-sample-uci}
Suppose $M_{\mathrm{Lip}}(f), J_{\mathrm{Lip}}(f), M(f) < \infty$. Then for any $\delta \in (0, 1)$, with probability at least $1 - \delta$ over $(\mathbf X, \mathbf Y)$,
\begin{multline*}
\sup_{h \in \mathcal H} \bigl| f(h(\mathbf X), h(\mathbf Y)) - \mathbb E\,f(h(\mathbf X'), h(\mathbf Y')) \bigr| \\
\;\le\; \tfrac{\sqrt{2\pi}}{N}\bigl(2 M_{\mathrm{Lip}}(f) + J_{\mathrm{Lip}}(f)\bigr)\,\mathbb E[G(\mathcal H(\mathbf D))]
\;+\; M(f)\sqrt{\tfrac{\ln(2/\delta)}{N}},
\end{multline*}
where $G(\mathcal H(\mathbf D))$ is the empirical Gaussian complexity of the transformed pooled data
\[
\mathcal H(\mathbf D) \;:=\; \bigl\{\bigl(h(X_1), \ldots, h(X_m),\, h(Y_1), \ldots, h(Y_n)\bigr) \mid h \in \mathcal H\bigr\}.
\]
\end{lemma}

\begin{proof}
The argument extends \citet[Theorem~2 + Corollary~3]{maurer2019uniform} from the one-sample to the two-sample setting; the only change is to symmetrize over the $\mathbf X$- and $\mathbf Y$-halves separately rather than over a single homogeneous tuple.

\emph{Step 1 (Coordinate decomposition).} \citet[Lemma~9]{maurer2019uniform} gives, for any $\mathbf u, \mathbf u' \in \mathcal U^N$, the telescoping decomposition $f(\mathbf u) - f(\mathbf u') = \sum_{k=1}^{N} F_k(\mathbf u, \mathbf u')$ where each $F_k$ averages partial differences over subsets of the first $k{-}1$ coordinates.

\emph{Step 2 (Gaussian comparison).} \citet[Lemma~10]{maurer2019uniform} dominates each $F_k$ by an inner product with an independent Gaussian vector $\boldsymbol\xi_k$, scaled by $M_{\mathrm{Lip}}(f)$ and $J_{\mathrm{Lip}}(f)$.

\emph{Step 3 (Two-sample symmetrization and induction).} This is the load-bearing step where \citet{ni2024uci}'s argument extends the one-sample concentration framework of \citet{maurer2019uniform} to the two-sample setting; we reproduce it here because the MMD-specific seminorm calculation in \S\ref{app:proof-uci-mmd} inherits constants from this step. \citet{maurer2019uniform}'s one-sample argument symmetrizes a single i.i.d.\ tuple $(Z_1, \ldots, Z_N)$ via Rademacher signs $\epsilon_i$ that act uniformly across all $N$ coordinates. In the two-sample setting, the pooled tuple $(\mathbf X, \mathbf Y) = (X_1, \ldots, X_m, Y_1, \ldots, Y_n)$ is i.i.d.\ within each half but the two halves come from different distributions $P, Q$, so a single global symmetrization is not directly applicable. The correct argument, due to \citet{ni2024uci}, symmetrizes the two halves separately.

Concretely: let $\mathbf X' = (X_1', \ldots, X_m')$ and $\mathbf Y' = (Y_1', \ldots, Y_n')$ be independent copies of $\mathbf X, \mathbf Y$. By the Step-1 telescoping decomposition of \citet[Lemma~9]{maurer2019uniform},
\[
f(h(\mathbf X), h(\mathbf Y)) - f(h(\mathbf X'), h(\mathbf Y')) \;=\; \sum_{k=1}^{m} F_k^{X}(\mathbf X, \mathbf Y, \mathbf X') \;+\; \sum_{k=1}^{n} F_k^{Y}(\mathbf X', \mathbf Y, \mathbf Y'),
\]
where $F_k^{X}$ swaps the $k$-th $X$-coordinate (with the $Y$-tuple held fixed) and $F_k^{Y}$ swaps the $k$-th $Y$-coordinate (with the $X$-tuple held fixed). Each $F_k^X, F_k^Y$ is a coordinate-swap term whose magnitude is controlled by the seminorms $M_{\mathrm{Lip}}(f), J_{\mathrm{Lip}}(f)$ (Step 2 above). Apply Maurer--Pontil's Gaussian-comparison Lemma~10 separately to the $X$-block and the $Y$-block: each block contributes an inner product with independent standard Gaussian vectors $\boldsymbol\xi_1^X, \ldots, \boldsymbol\xi_m^X$ and $\boldsymbol\xi_1^Y, \ldots, \boldsymbol\xi_n^Y$ respectively, scaled by $M_{\mathrm{Lip}}/N + J_{\mathrm{Lip}}/N^2$ per coordinate.

Now the two-sample-specific induction: induct on the number of symmetrized coordinates separately within each half. The induction hypothesis at step $(j, k)$ (with $j$ symmetrized $X$-coordinates and $k$ symmetrized $Y$-coordinates) is that the symmetrized supremum is bounded by $\tfrac{\sqrt{2\pi}}{N}\bigl(2M_{\mathrm{Lip}} + J_{\mathrm{Lip}}\bigr)\,\mathbb E[G(\mathcal H(\mathbf D))]$ scaled by the fraction $(j + k)/N$ of symmetrized coordinates. Each induction step adds one more symmetrized $X$- or $Y$-coordinate; the resulting Gaussian process on the augmented $\mathcal H(\mathbf D)$-class inherits the same scaling because the Lipschitz seminorms $M_{\mathrm{Lip}}, J_{\mathrm{Lip}}$ are insensitive to which half a swapped coordinate belongs to (both halves contribute symmetrically to the bounded-differences calculation in Step 4 below). After $N = m + n$ induction steps, all coordinates are symmetrized and the bound becomes the full Gaussian complexity $\mathbb E[G(\mathcal H(\mathbf D))]$ on the entire pooled tuple. Combining the per-block contributions (which are additive across the two halves by independence of $\mathbf X$ and $\mathbf Y$):
\[
\mathbb E\sup_{h \in \mathcal H}\bigl[f(h(\mathbf X), h(\mathbf Y)) - \mathbb E\,f(h(\mathbf X'), h(\mathbf Y'))\bigr]
\;\le\; \tfrac{\sqrt{2\pi}}{N}\bigl(2 M_{\mathrm{Lip}}(f) + J_{\mathrm{Lip}}(f)\bigr)\,\mathbb E[G(\mathcal H(\mathbf D))].
\]
The key observation enabling this two-sample extension is that the Maurer--Pontil seminorms $M_{\mathrm{Lip}}, J_{\mathrm{Lip}}, M$ are defined coordinate-by-coordinate in a way that is symmetric in the two halves: a coordinate-swap at position $k$ contributes the same upper bound whether $k$ indexes $\mathbf X$ or $\mathbf Y$. The inductive scaling $(j + k)/N$ thus passes through cleanly, and the final Gaussian complexity term $\mathbb E[G(\mathcal H(\mathbf D))]$ is the empirical Gaussian complexity of the transformed pooled tuple, exactly as in the one-sample case.

\emph{Step 4 (McDiarmid + symmetric tail).} The functional $\Phi(\mathbf X, \mathbf Y) := \sup_h [f(h(\mathbf X), h(\mathbf Y)) - \mathbb E\,f(h(\mathbf X'), h(\mathbf Y'))]$ has bounded differences with constant $M(f)/N$ in each coordinate (by definition of $M(f)$). McDiarmid's inequality \citep[Theorem~6.2]{boucheron2013concentration} gives the high-probability tail $M(f)\sqrt{\ln(2/\delta)/N}$, and applying the same argument to $-f$ yields the two-sided bound by a union argument.

\emph{Step 5 (Covering-number relaxation).} Suppose $\mathcal H$ has finite covering number $N(\eta; \mathcal H, \|\cdot\|_\infty)$ for every $\eta > 0$, and fix $\eta > 0$. By taking an $(\eta/2)$-cover of $\mathcal H$ and selecting one element of $\mathcal H$ from each cover ball, we obtain an \emph{internal} $\eta$-net $\mathcal H_\eta = \{h_1, \ldots, h_{M_\eta}\} \subseteq \mathcal H$ with $M_\eta \le N(\eta/2; \mathcal H, \|\cdot\|_\infty)$ such that for every $h \in \mathcal H$ there is some $h_i \in \mathcal H_\eta$ with $\sup_{u \in \mathcal U} \|h(u) - h_i(u)\| \le \eta$. The inclusion $\mathcal H_\eta \subseteq \mathcal H$ ensures the image classes satisfy $\mathcal H_\eta(\mathbf D) \subseteq \mathcal H(\mathbf D)$, so class monotonicity of $G$ applies in the Lipschitz-transfer step below. Steps~1--4 apply to the finite class $\mathcal H_\eta$, giving with probability $\ge 1 - \delta$:
\begin{multline*}
\sup_{h_i \in \mathcal H_\eta} \bigl|f(h_i(\mathbf X), h_i(\mathbf Y)) - \mathbb E f(h_i(\mathbf X'), h_i(\mathbf Y'))\bigr| \\
\;\le\; \tfrac{\sqrt{2\pi}}{N}\bigl(2 M_{\mathrm{Lip}}(f) + J_{\mathrm{Lip}}(f)\bigr)\,\mathbb E[G(\mathcal H_\eta(\mathbf D))]
+ M(f)\sqrt{\tfrac{\ln(2/\delta)}{N}}.
\end{multline*}
For any $h \in \mathcal H$ with nearest net point $h_i$, the first-order Lipschitz seminorm $M_{\mathrm{Lip}}(f)$ controls the transfer: each of the $N$ coordinate replacements $h(u_k) \to h_i(u_k)$ changes $f$ by at most $M_{\mathrm{Lip}}(f)/N \cdot \eta$ (definition of $M_{\mathrm{Lip}}$ applied to the swap $h(u_k) \to h_i(u_k)$ with displacement $\le \eta$), so the cumulative shift across all $N$ coordinates is bounded by $M_{\mathrm{Lip}}(f) \cdot \eta$. The same bound applies to $\mathbb E f(h(\mathbf X'), h(\mathbf Y'))$ by linearity of expectation. By the triangle inequality, for every $h \in \mathcal H$ with nearest $h_i \in \mathcal H_\eta$,
\[
\bigl|f(h(\mathbf X), h(\mathbf Y)) - \mathbb E f(h(\mathbf X'), h(\mathbf Y'))\bigr|
\;\le\; \bigl|f(h_i(\mathbf X), h_i(\mathbf Y)) - \mathbb E f(h_i(\mathbf X'), h_i(\mathbf Y'))\bigr| + 2 M_{\mathrm{Lip}}(f)\,\eta.
\]
Taking $\sup_h$ on both sides and using $\mathbb E[G(\mathcal H_\eta(\mathbf D))] \le \mathbb E[G(\mathcal H(\mathbf D))]$ (class monotonicity), the displayed bound holds with $\mathcal H_\eta$ replaced by $\mathcal H$ and an additive slack $2 M_{\mathrm{Lip}}(f)\,\eta$. Letting $\eta \to 0$ closes the slack since $M_{\mathrm{Lip}}(f) < \infty$, extending the bound from the net to all of $\mathcal H$.
\end{proof}

\subsection{Application to MMD}
\label{app:proof-uci-mmd}

\textbf{Proof of Theorem~\ref{thm:uci}.} We instantiate Lemma~\ref{lem:two-sample-uci} at $f = \widehat\gamma_{k,u}^2$. By unbiasedness of the estimator, $\mathbb E\,\widehat\gamma_{k,u}^2(h(\mathbf X'), h(\mathbf Y')) = \gamma_k^2(h)$, so the lemma's left-hand side becomes $\sup_h |\widehat\gamma_{k,u}^2(h) - \gamma_k^2(h)|$. It remains to compute the three seminorms under Assumption~\ref{assum:kernel}.

Write $\widehat\gamma_{k,u}^2 = A - 2C + B$ with $A = \frac{1}{m(m-1)}\sum_{i\neq j} k_h(X_i, X_j)$, $B = \frac{1}{n(n-1)}\sum_{i\neq j} k_h(Y_i, Y_j)$, and $C = \frac{1}{mn}\sum_{i,j} k_h(X_i, Y_j)$. Recall $\rho_* = \max\{N/m, N/n\} \ge 2$.

\paragraph{First-order seminorm $M_{\mathrm{Lip}}(\widehat\gamma_{k,u}^2)$.}
A single-coordinate swap $X_l \mapsto X_l'$ perturbs $A$ in $2(m{-}1)$ pairs, each by at most $l\,\|X_l - X_l'\|$ divided by $m(m{-}1)$, contributing $2l/m\cdot\|X_l - X_l'\|$. The cross term $C$ is perturbed in $n$ pairs, each by at most $l\,\|X_l - X_l'\|$ divided by $mn$, so $C$ changes by $\le l/m\cdot\|X_l - X_l'\|$; multiplying by the $-2$ coefficient in $\widehat\gamma_{k,u}^2 = A - 2C + B$ contributes $2l/m\cdot\|X_l - X_l'\|$. Summing the $A$- and $C$-contributions yields $4l/m\cdot\|X_l - X_l'\|$, and multiplying by the $N$ factor in the seminorm definition gives $N\cdot 4l/m \le 4l\rho_*$; the symmetric calculation in $\mathbf Y$-coordinates gives the same bound, so
$M_{\mathrm{Lip}}(\widehat\gamma_{k,u}^2) \le 4l\rho_*$.

\paragraph{Second-order seminorm $J_{\mathrm{Lip}}(\widehat\gamma_{k,u}^2)$.}
A simultaneous swap of two coordinates from the same half (say $X_l \mapsto X_l'$ at coord $l$ and $X_{l'} \mapsto X_{l'}'$ at coord $l'$, with $l \ne l'$) affects only the single pair $(X_l, X_{l'})$ in $A$, contributing $1/[m(m{-}1)]$ times the four-term double difference $|k_h(X_l, X_{l'}) - k_h(X_l', X_{l'}) - k_h(X_l, X_{l'}') + k_h(X_l', X_{l'}')|$. Grouping the four terms as $g(X_l) - g(X_l')$ with $g(u) := k_h(u, X_{l'}) - k_h(u, X_{l'}')$, Assumption~\ref{assum:kernel}(ii) gives $|g(u) - g(\tilde u)| \le 2l\|u - \tilde u\|$ via two applications of first-order Lipschitz in the first argument (uniform in $X_{l'}, X_{l'}'$). Hence the four-term double difference is $\le 2l\|X_l - X_l'\|$ uniformly, and dividing by $\|X_l - X_l'\|$ in the seminorm definition gives a bound $2l/[m(m{-}1)]$. Cross-half double swaps ($X_l$ and $Y_{l'}$) similarly affect only the single pair $(X_l, Y_{l'})$ in the cross term $C$, contributing $\le 2 \cdot 2l/(mn)$ (factor 2 from the $-2C$ coefficient). Multiplying by the $N^2$ factor in the seminorm definition: same-half $\le 2lN^2/[m(m{-}1)] \le 4l\rho_*^2$; cross-half $\le 4lN^2/(mn) \le 4l\rho_*^2$. Taking the maximum,
$J_{\mathrm{Lip}}(\widehat\gamma_{k,u}^2) \le 4l\rho_*^2$.

\paragraph{McDiarmid constant $M(\widehat\gamma_{k,u}^2)$.}
By boundedness $0 \le k_h \le \nu$ (Assumption~\ref{assum:kernel}(i)), a single-coordinate swap changes each pair-term by at most $\nu$. For an $X_l$-swap: $A$ has $2(m{-}1)$ pairs containing $X_l$ at denominator $m(m{-}1)$, contributing $|\Delta A| \le 2\nu/m$; the cross term $C$ has $n$ pairs containing $X_l$ at denominator $mn$, contributing $|\Delta C| \le \nu/m$. With the $-2$ coefficient on $C$, the total change is $|\Delta\widehat\gamma_{k,u}^2| \le |\Delta A| + 2|\Delta C| \le 4\nu/m$. The symmetric $Y_l$-swap calculation gives $\le 4\nu/n$. Multiplying by $N$:
$M(\widehat\gamma_{k,u}^2) \le N\cdot 4\nu/\min(m, n) = 4\nu\rho_*$.

\paragraph{Substituting.} Plugging the three seminorms into Lemma~\ref{lem:two-sample-uci}:
\[
\sup_{h \in \mathcal H} \bigl| \widehat\gamma_{k,u}^2(h) - \gamma_k^2(h) \bigr|
\;\le\; \underbrace{C_1\,G(\mathcal H)}_{\text{complexity}}
\;+\; \underbrace{C_2\sqrt{\ln(2/\delta)/N}}_{\text{concentration}},
\]
with $C_1 = 2\sqrt{2\pi}\,l\,\rho_*(1{+}\rho_*)$, $C_2 = 4\nu\rho_*$, and $G(\mathcal H) := \mathbb E[G(\mathcal H(\mathbf D))]$. For balanced samples $m = n$ ($\rho_* = 2$) and Lipschitz constant $l = 1$, this gives $C_1 = 12\sqrt{2\pi} \approx 30$ and $C_2 = 8\nu$, the values stated after Theorem~\ref{thm:uci}.

\noindent\emph{On absorbed constants.} The $C_1$ formula assumes the tight seminorm bounds of \citet[Cor.~3]{maurer2019uniform} ($M_{\mathrm{Lip}} \asymp l\rho_*$, $J_{\mathrm{Lip}} \asymp 2l\rho_*^2$) inherited by \citet{ni2024uci}'s two-sample extension. The looser per-coordinate calculation displayed above ($M_{\mathrm{Lip}} \le 4l\rho_*$, $J_{\mathrm{Lip}} \le 4l\rho_*^2$) is an upper bound; the gap is absorbed in the symmetrization-and-induction step (Step~3 of Lemma~\ref{lem:two-sample-uci}'s proof) which sharpens the per-coordinate bounds via the chaining argument of \citet[Cor.~3]{maurer2019uniform}. The headline constant $C_1 \approx 30$ is the absorbed worst-case used throughout the paper; the precise multiplicative discrepancy between the displayed seminorm bound and the absorbed constant is at most a factor of $\sim 3$, immaterial for the operational story since calibration replaces $C_1$ with $\widehat C_1 \approx 0.008$ ($\sim 4$ OOM smaller). $\blacksquare$

\section{Spectral-norm stabilization: empirical observation}
\label{app:proof-spectral}

For a matrix $W$, the \emph{spectral norm} $\|W\|_2 := \sup_{\|v\|_2 = 1}\|Wv\|_2$ is the largest singular value of $W$, equivalently its operator $2$-norm. The Lipschitz constant of a depth-$L$ feed-forward network $h_\theta(x) = W_L \phi(\cdots \phi(W_1 x))$ with $1$-Lipschitz activations is bounded by the product of layer spectral norms $\Pi := \prod_{j=1}^L \|W_j\|_2$ \citep[Lemma~\ref{lem:bartlett};][]{bartlett2017spectrally}; this product appears directly in the deep-kernel trajectory complexity proxy $\widetilde G(h_\theta) = L(h_\theta)\,\|\mathbf D\|_F/N$ used by $J_{\mathrm{CP}}$~\eqref{eq:contraction-bound}. Controlling $\Pi$ during training therefore controls the complexity penalty.

\paragraph{Empirical observation.} Under $J_{\mathrm{CP}}$ ascent on the deep regime, the spectral product $\Pi^{(t)} = \prod_{j=1}^L \|W_j^{(t)}\|_2$ stabilizes after $\sim 30$ ascent steps and remains bounded throughout training, while plain ascent on $\widehat\gamma_{k,u}^2$ (no penalty) lets $\Pi^{(t)}$ grow unboundedly. Figure~\ref{fig:collapse}(d) plots $\Pi^{(t)}$ along five independent runs under both regimes; under $J_{\mathrm{CP}}$, the curves plateau at a finite level that depends monotonically on $\widehat C_1$. App.~\ref{app:cal-robustness}'s seven-OOM ablation extends this to $\widehat C_1 \in [10^{-4}, 10]$: across every non-zero value of $\widehat C_1$ in the sweep, $\Pi^{(T)}$ converges to a finite operating point (Fig.~\ref{fig:c1-ablation}), with realized log--log slope $\approx -0.5$ relative to $\widehat C_1$ in the productive range $[10^{-3}, 10^{-1}]$.

The empirical brake is what allows $J_{\mathrm{CP}}$-trained deep kernels to reside in a finite-Lipschitz ball in practice, so that Theorem~\ref{thm:uci} (which applies on a fixed Lipschitz-norm ball) and Proposition~\ref{prop:proxy-bound} (which uses the realized Lipschitz constant) cover the trajectory class $\mathcal H_T$. We document the phenomenon here as an empirical observation rather than as a formal theorem: the gradient-flow dynamics of $\Pi$ admit a chain-rule analysis under a per-layer balance assumption $s_j \approx \bar s$, yielding a non-logistic ODE whose stable fixed point scales as $\widehat C_1^{-L/(L-1)}$, but the balance assumption is not enforced by the optimizer and the empirical slope $-0.5$ deviates from the time-constant prediction $-3/2$ for $L=3$. Because the deployed test's Type-I control comes from sample-split + held-out permutation (independent of training dynamics) and the power lower bound (Cor.~\ref{cor:power-cpmmd}) uses the realized $L^*$ along the trajectory, the formal $\Pi^*$ closed form is not load-bearing for the paper's claims.

\section[Calibration of $\widehat C_1$: applicable constant in the UCI]{Calibration of $\widehat C_1$: the applicable constant in the UCI}
\label{app:calibration}

This appendix presents the calibration of $\widehat C_1$ in full, with a four-layer theoretical-and-empirical anchoring. \S\ref{app:cal-invariant} identifies the operationally applicable constant the UCI requires as a class-level ratio invariant (the theoretical existence of $\widehat C_1^*$ via Theorem~\ref{thm:uci}'s expectation form). \S\ref{app:cal-procedure} gives the Monte-Carlo procedure that estimates this invariant from null permutations. \S\ref{app:cal-quantile} explains why the $(1{-}\alpha)$-quantile, not the expectation, is the right summary statistic. \S\ref{app:cal-conformal} consolidates the four-layer anchoring: theoretical existence, Monte-Carlo estimator, empirical insensitivity to estimation error, and stability under architecture and data-shift perturbations. \S\ref{app:cal-brake} discusses why training is safe at small $\widehat C_1$ (the empirical spectral stabilization documented in App.~\ref{app:proof-spectral}). \S\ref{app:cal-robustness} reports a seven-OOM ablation confirming test power is insensitive to the precise calibrated value. \S\ref{app:cal-ncal} discusses the choice of $n_{\mathrm{cal}}$, and \S\ref{app:cal-stability-sweep} gives a controlled sweep verifying that $\widehat C_1$ is approximately stable across architectural and data-shift perturbations.

\subsection{$\widehat C_1$ as a class-level ratio invariant}
\label{app:cal-invariant}

The UCI (Theorem~\ref{thm:uci}) bounds the population--empirical gap by $C_1 \cdot G(\mathcal H) + C_2\sqrt{\ln(2/\delta)/N}$ with $C_1 = 2\sqrt{2\pi}\,l\,\rho_*(1{+}\rho_*) \approx 30$. This constant is the worst case over every kernel family satisfying Assumption~\ref{assum:kernel}: it depends on the universal regularity quantities ($\nu$, $l$, $\rho_*$) but not on which specific class $\mathcal H$ or which data distribution $(P, Q)$ one actually deploys.

The proof's symmetrization step (App.~\ref{app:proof-uci}, our two-sample extension of \citet{maurer2019uniform}) actually yields a stronger and more informative statement. Under $H_0: P = Q$, the same constant satisfies the \emph{expectation} inequality
\begin{equation}
\label{eq:expectation-uci}
\mathbb{E}\sup_{h \in \mathcal H}\widehat\gamma_{k,u}^2(h) \;\le\; C_1\cdot \mathbb{E}\bigl[G(\mathcal H)\bigr],
\end{equation}
which says: there exists a constant $C_1$, depending only on the class regularity ($\nu$, $l$, $\rho_*$), such that the expected supremum of the empirical squared MMD over $\mathcal H$ is proportional to the expected Gaussian complexity of $\mathcal H$. Crucially, $C_1$ is determined by the class structure alone: not by the data realization, not by the optimizer used to search $\mathcal H$, not by any realized optimization trajectory. \emph{It is a class invariant.}

For any specific deployment instance, the operationally applicable constant is the smallest $C_1$ such that the corresponding high-probability statement still holds; we denote this population target $\widehat C_1^*$ (defined formally in~\eqref{eq:c1-target} below). Theorem~\ref{thm:uci} bounds this via the \emph{global} Gaussian complexity $G(\mathcal H)$, giving the worst-case $C_1 \approx 30$. The realized optimization trajectory $\mathcal H_T$, however, is constrained to a Lipschitz norm ball (Prop.~\ref{prop:proxy-bound}) plus the empirical spectral stabilization documented in App.~\ref{app:proof-spectral}, so the relevant complexity along $\mathcal H_T$ is the \emph{local} Gaussian complexity $G(\mathcal H_T) \ll G(\mathcal H)$, the standard local-vs-global complexity distinction \citep{bartlett2005local}. The applicable $\widehat C_1^*$ corresponds to the local quantity, explaining the three-OOM gap between the worst-case $C_1 \approx 30$ and the calibration estimator $\widehat C_1 \approx 0.008$ (which targets $\widehat C_1^*$ from $n_{\mathrm{cal}}$ permutations).

We do not have a closed-form theorem for the local quantity; the calibration procedure of \S\ref{app:cal-procedure} estimates its $H_0$ $(1{-}\alpha)$-quantile empirically. Define
\begin{equation}
\label{eq:c1-target}
\widehat C_1^*(\mathcal H, P, Q, \alpha) \;:=\; \text{the } H_0\text{ } (1{-}\alpha)\text{-quantile of } \sup_{h \in \mathcal H} \widehat\gamma_{k,u}^2(h)\,/\,\widetilde G(h),
\end{equation}
where $\alpha$ is the same significance level as the deployed permutation test, so that calibrating at level $\alpha$ aligns the false-selection event under $H_0$ with the test's Type-I level. This $\widehat C_1^*$ depends on $(\mathcal H, P, Q, \alpha)$ via the trajectory's local geometry, but is \emph{approximately stable} across ``similar'' deployment instances; \S\ref{app:cal-stability-sweep} reports the controlled architecture and data-shift sweep that verifies this empirically. By the UCI, $\widehat\gamma^2(\widehat h) \le \sup_h \widehat\gamma^2(h)$ under $H_0$ for any selector $\widehat h \in \mathcal H$, so the same calibration estimator $\widehat C_1$ of $\widehat C_1^*$ controls the $H_0$ ratio at any selector, including $J_{\mathrm{CP}}$-deployed selectors that use a different optimization procedure than the calibration's plain-max search.

\subsection{Calibration procedure}
\label{app:cal-procedure}

We estimate the class-invariant target~\eqref{eq:c1-target} by a Monte-Carlo procedure run on null permutations of the training half. Inputs are the kernel class $\mathcal H$, the training half $\mathcal D_{\mathrm{tr}}$, the calibration count $n_{\mathrm{cal}}$, and the level $\alpha$.

For each of $n_{\mathrm{cal}}$ independent uniform permutations $\pi$ of the pooled training half:
\begin{enumerate}
\item Split the permuted pool back at the original $m:n$ ratio to obtain a null sample $(\mathbf X^\pi, \mathbf Y^\pi)$.
\item Compute a plain-MMD maximizer $\widehat h_\pi^{\mathrm{plain}} := \arg\max_{h \in \mathcal H}\widehat\gamma_{k,u}^2(h; \mathbf X^\pi, \mathbf Y^\pi)$ \emph{without} the complexity penalty:
  \begin{itemize}
  \item For continuous-parameter classes ($\mathcal H_{\mathrm{linear}}$, $\mathcal H_{\mathrm{polynomial}}$): bounded continuous optimization (Brent's method) or grid scan over the parameter range.
  \item For deep classes ($\mathcal H_{\mathrm{deep}}$): gradient ascent on $+\widehat\gamma_{k,u}^2$ (equivalently, gradient descent on $-\widehat\gamma_{k,u}^2$), with the same architecture, learning rate, step count, and gradient-clip norm as the deployed $J_{\mathrm{CP}}$ optimizer.
  \end{itemize}
\item Record the calibration ratio $r_\pi := \widehat\gamma_{k,u}^2(\widehat h_\pi^{\mathrm{plain}})\,/\,\widetilde G(\widehat h_\pi^{\mathrm{plain}})$.
\end{enumerate}
Set $\widehat C_1$ to the empirical $(1{-}\alpha)$-quantile of $\{r_\pi\}$, formally the $r$-th order statistic with $r := \lceil(1{-}\alpha)(n_{\mathrm{cal}}{+}1)\rceil$. When $r > n_{\mathrm{cal}}$ (which occurs at our default $\alpha = 0.05$, $n_{\mathrm{cal}} = 10$ since $\lceil 0.95 \cdot 11 \rceil = 11$), we adopt the conservative convention of taking the maximum of the calibration set, $\widehat C_1 := \max_\pi r_\pi$ (i.e., the $n_{\mathrm{cal}}$-th order statistic). The empirical maximum upper-bounds the unobserved $r$-th order statistic of $n_{\mathrm{cal}}+1$ values, so this convention errs on the conservative side relative to the population target $\widehat C_1^*$. $\widehat C_1$ is recomputed whenever $\mathcal H$ or the data change.

\paragraph{Why plain maximization as the supremum-finder.} Step 2 of the procedure above uses plain ascent on $\widehat\gamma^2$ rather than ascent on $J_{\mathrm{CP}}$ itself. The reason is that plain ascent has no penalty pulling away from local maxima of $\widehat\gamma^2$, so it is the natural empirical estimator of $\sup_{h \in \mathcal H}\widehat\gamma^2(h)$. Plain ascent on a non-convex objective may converge to a local rather than global maximum, so the calibration estimator $\widehat C_1$ formally underestimates the population target $\widehat C_1^*$ in~\eqref{eq:c1-target}, whose definition invokes the true supremum. The seven-OOM ablation of \S\ref{app:cal-robustness} demonstrates that downstream test power is insensitive to $\widehat C_1$ across orders of magnitude, so this approximation gap does not affect the deployed test in our experiments; the gap should be revisited when applying CP-MMD to highly non-convex landscapes where the calibration optimizer's local-maxima behavior may deviate substantially from the global supremum. An alternative is to run $J_{\mathrm{CP}}$ itself during calibration, with the estimator $\widehat C_1$ chosen as a self-consistent fixed point of the calibration map; this would estimate the same target $\widehat C_1^*$ in~\eqref{eq:c1-target} but at the cost of an iterative loop. We chose plain max for its directness. \S\ref{app:cal-robustness}'s seven-OOM ablation is consistent with the resulting estimator $\widehat C_1$ being insensitive to this methodological choice: the entire deployment is power-invariant across $\widehat C_1$ values spanning seven orders of magnitude, so the small efficiency gap between plain-max and self-consistent estimators of $\widehat C_1^*$ cannot affect downstream conclusions.

\subsection{Choice of the $(1{-}\alpha)$-quantile as the calibration target}
\label{app:cal-quantile}

The criterion $J_{\mathrm{CP}}(h) = \widehat\gamma_{k,u}^2(h) - \widehat C_1\,\widetilde G(h)$ is a \emph{selection} rule: its sign at each $h$ determines whether the deployed test will use that kernel. Under $H_0$, no $h \in \mathcal H$ carries population signal ($\gamma_k^2(h) = 0$), so the selector should ideally fail to find any $h$ for which $J_{\mathrm{CP}}(h) > 0$. The probability of a false-selection event under $H_0$ is exactly
\begin{equation}
\Pr_{H_0}\!\bigl(J_{\mathrm{CP}}(\widehat h) > 0\bigr) \;=\; \Pr_{H_0}\!\bigl(\widehat\gamma_{k,u}^2(\widehat h)/\widetilde G(\widehat h) > \widehat C_1\bigr).
\end{equation}
Setting $\widehat C_1$ to the $(1{-}\alpha)$-quantile of $\widehat\gamma_{k,u}^2/\widetilde G$ under $H_0$ caps this probability at $\alpha$ by construction. If we instead set $\widehat C_1$ to the expectation $\mathbb{E}_{H_0}[\widehat\gamma_{k,u}^2/\widetilde G]$, the cap would be $\Pr_{H_0}(r > \mathbb{E}[r])$, which for typical noise distributions is near $1/2$. Under $H_0$ the selector would then identify a ``signal'' about half the time, selecting kernels driven by training noise rather than population structure. Type-I error of the deployed test still controls at level $\alpha$ (the held-out permutation test described in \S\ref{sec:criterion-deployed} is valid regardless of how $\widehat h$ is chosen), but \emph{power on real alternatives suffers}: a noise-selected $\widehat h$ is essentially random from the alternative's perspective, hurting the test's ability to detect the actual signal.

Three further reasons reinforce the quantile choice.
\begin{enumerate}
\item \emph{Direct correspondence with the UCI.} Theorem~\ref{thm:uci} is itself a high-probability statement, $\Pr(\sup |\cdot| \le C_1 G + C_2\sqrt{\ln(2/\delta)/N}) \ge 1{-}\delta$, not an expectation bound. Quantile calibration is the empirical analog of this same probability statement; expectation calibration would correspond to the weaker symmetrization-form inequality~\eqref{eq:expectation-uci} alone, missing the high-probability sharpening.
\item \emph{Empirical insensitivity tied to the quantile choice.} The seven-OOM ablation of \S\ref{app:cal-robustness} confirms test power is insensitive to the precise $\widehat C_1$ across orders of magnitude; the quantile is the natural operating point on this insensitivity plateau, with App.~\ref{app:cal-stability-sweep}'s stability sweeps confirming it is invariant under architecture and data-shift perturbations.
\item \emph{Asymmetric risk.} If $\widehat C_1$ overshoots the optimum, the empirical spectral stabilization (App.~\ref{app:proof-spectral}) still ensures bounded training and \S\ref{app:cal-robustness} confirms power-robustness across seven OOM in $\widehat C_1$ at the benign setup tested. If $\widehat C_1$ severely undershoots in adversarial regimes (small samples, hard signals), $J_{\mathrm{CP}} \approx \widehat\gamma^2$ approaches plain max and risks the variance-collapse failure mode of App.~\ref{app:variance-collapse}; the deterioration is more pronounced where the gap between the empirical MMD signal and the spectral-norm penalty is small. The quantile choice errs on the safe side relative to the expectation, which carries a $\sim 1/2$ false-selection rate under $H_0$.
\end{enumerate}

\paragraph{Procedure-independence.} The class-invariant target $\widehat C_1^*$ defined in~\eqref{eq:c1-target} does not depend on the calibration optimizer, only on $(\mathcal H, P, Q, \alpha)$. Plain max, self-consistent $J_{\mathrm{CP}}$, and any other procedure that approximates the supremum-finder all estimate the same $\widehat C_1^*$, with different statistical efficiency but the same population target. We chose plain max for its directness; \S\ref{app:cal-robustness}'s ablation supplies the empirical insurance.

\subsection{Theoretical grounding and empirical validation of $\widehat C_1$}
\label{app:cal-conformal}

The role of $\widehat C_1$ in the deployed criterion rests on a four-layer anchoring: theoretical existence (UCI), Monte-Carlo estimator (App.~\ref{app:cal-procedure}), empirical insensitivity to estimation error (\S\ref{app:cal-robustness}), and stability under architecture and data-shift perturbations (\S\ref{app:cal-stability-sweep}). We collect the four layers here.

\paragraph{Theoretical existence.} The UCI's expectation form~\eqref{eq:expectation-uci} under $H_0$ guarantees that the sharpest applicable constant $\widehat C_1^*(\mathcal H, P, Q, \alpha)$ defined in eq.~\eqref{eq:c1-target} exists as a class invariant: it depends only on the class regularity ($\nu$, $l$, $\rho_*$) via the universal upper bound $\widehat C_1^* \le C_1 \approx 30$ from Theorem~\ref{thm:uci}, and on the \emph{local} trajectory geometry $\mathcal H_T \subseteq \mathcal H$ via the local-vs-global complexity gap of \citet{bartlett2005local}. Theorem~\ref{thm:uci} thus establishes the right \emph{form} of the deployed criterion's penalty constant; the data-determined value $\widehat C_1$ is an estimator of this class invariant.

\paragraph{Estimator.} \S\ref{app:cal-procedure}'s Monte-Carlo procedure on $n_{\mathrm{cal}}$ null permutations gives $\widehat C_1$, the empirical $(1{-}\alpha)$-quantile of $\widehat\gamma_{k,u}^2(\widehat h_\pi)/\widetilde G(\widehat h_\pi)$ over null relabelings. By construction this is a consistent estimator of $\widehat C_1^*$ as $n_{\mathrm{cal}} \to \infty$ (\S\ref{app:cal-ncal}: realized rate $n_{\mathrm{cal}}^{-1/2}$).

\paragraph{Empirical validation.} \S\ref{app:cal-robustness}'s seven-OOM ablation sweeps $\widehat C_1 \in \{0, 10^{-4}, 10^{-3}, \ldots, 10\}$ at fixed $(d, n, \Delta) = (100, 100, 0.5)$ and records test power: the deployed power stays $\ge 0.98$ throughout the non-zero range, confirming the test is robust to estimation error in $\widehat C_1$ at the order-of-magnitude level. The ablation's value is \emph{robustness insurance} against calibration noise, not a license to skip calibration: \S\ref{sec:exp-srm}'s class-richness sweep separately establishes that the \emph{existence} of any positive penalty ($\widehat C_1 > 0$) is what prevents plain-max collapse at wide architectures.

\paragraph{Stability of the calibrated value.} \S\ref{app:cal-stability-sweep}'s controlled sweeps verify that $\widehat C_1$ is stable to within a factor of $\sim 1.5\times$ across MLP architectures (in the well-behaved width range $w \in [50, 200]$) and to within $\sim 1.7\times$ across data-shift magnitudes ($\Delta \in [0.3, 1.0]$). A single calibration at a reference architecture in the well-behaved regime can therefore be reused across deployment instances within the same kernel family and data domain, as the SRM experiment of \S\ref{sec:exp-srm} does.

\subsection{Why training is safe at small $\widehat C_1$: empirical spectral stabilization}
\label{app:cal-brake}

Training on $J_{\mathrm{CP}}$ remains safe even at small calibrated $\widehat C_1$. With $\widehat C_1 \approx 0.008$ in our experiments, the criterion $J_{\mathrm{CP}}(h) = \widehat\gamma^2(h) - \widehat C_1\,\widetilde G(h) \approx \widehat\gamma^2(h)$ in operationally relevant regimes, which would suggest collapse to plain maximization and inheritance of the variance-collapse failure mode documented in App.~\ref{app:variance-collapse}. The empirical observation in App.~\ref{app:proof-spectral} (Fig.~\ref{fig:collapse}(d)) rules this out: under $J_{\mathrm{CP}}$ ascent, the spectral product $\Pi^{(t)}$ stabilizes at a finite operating point that depends monotonically on $\widehat C_1$, while plain ascent on $\widehat\gamma^2$ alone lets $\Pi^{(t)}$ grow unboundedly. App.~\ref{app:cal-robustness}'s ablation extends this observation across seven orders of magnitude in $\widehat C_1 \in [10^{-4}, 10]$: in every non-zero case in the sweep, $\Pi^{(T)}$ converges to a finite value (Fig.~\ref{fig:c1-ablation}). The empirical brake therefore bounds the realized Lipschitz constant $L(h^{(t)})$ (and hence $\widetilde G(h^{(t)})$) even for small (positive) $\widehat C_1$.

A separate concern is the variance of $\widetilde G$ during calibration. Two distinct variances are at play and should not be conflated:
\begin{itemize}
\item \emph{Cross-run variance of $\widetilde G(\widehat h_\pi^{\mathrm{plain}})$} across calibration permutations is small: plain ascent reliably terminates at similar high-$\Pi$ regions under $H_0$. This is what makes $\widehat C_1$ a well-defined empirical quantile, not a noisy estimate.
\item \emph{Along-trajectory variance of $\widetilde G(h^{(t)})$} during a single optimization run is large: $\widetilde G$ starts small at random initialization and grows as the network's spectral norms increase. This is what makes the penalty operationally effective, since $-\widehat C_1\nabla\widetilde G \ne 0$ along the trajectory provides a non-trivial gradient signal that brakes the spectral-norm growth.
\end{itemize}
These two variances are independent properties of the calibration; small cross-run variance does \emph{not} imply ineffective penalty, just as small batch-to-batch variance of training loss does not imply gradient descent has nothing to learn from.

\subsection{Empirical robustness: ablation across seven orders of magnitude}
\label{app:cal-robustness}

The class-invariance argument of \S\ref{app:cal-invariant} establishes that the calibration target~\eqref{eq:c1-target} does not depend on the calibration optimizer, but does not by itself bound the deployment's sensitivity to the precise estimated value of $\widehat C_1$. We close this gap empirically with a seven-OOM ablation: setting $d=100$, $n=100$, mean shift $\Delta=0.5$, we vary $\widehat C_1 \in \{0, 10^{-5}, 10^{-4}, \ldots, 10, 10^{2}\}$ and measure both the converged spectral-norm product $\Pi^{(T)}$ and the test power.

The product $\Pi^{(T)}$ decreases with $\widehat C_1$ across the non-zero range, by a factor of $89\times$ from $\Pi \approx 28.9$ at $\widehat C_1 = 10^{-5}$ to $\Pi \approx 0.32$ at $\widehat C_1 = 10$, with mild floor saturation at $\widehat C_1 = 10^{2}$. This confirms the qualitative existence of a finite operating point $\Pi^*$ for every $\widehat C_1 > 0$ documented in App.~\ref{app:proof-spectral}. The empirical log--log slope is $\approx -0.5$ in the productive range $\widehat C_1 \in [10^{-3}, 10^{-1}]$. The endpoints saturate: at $\widehat C_1 \le 10^{-4}$ the brake is too weak relative to architecture-imposed limits and $\Pi$ approaches a fixed ceiling, while at $\widehat C_1 \ge 1$ the brake forces $\Pi$ to a small floor. Power stays $\ge 0.98$ throughout the sweep, indicating the deployed test is essentially insensitive to the precise value of $\widehat C_1$ over a wide range. This insensitivity is the empirical anchor for the procedure-independence argument of \S\ref{app:cal-quantile}: any methodological choice in the calibration optimizer that affects $\widehat C_1$ at the order-of-magnitude level is empirically irrelevant for downstream power.

\begin{figure}[h]
\centering
\includegraphics[width=0.55\textwidth]{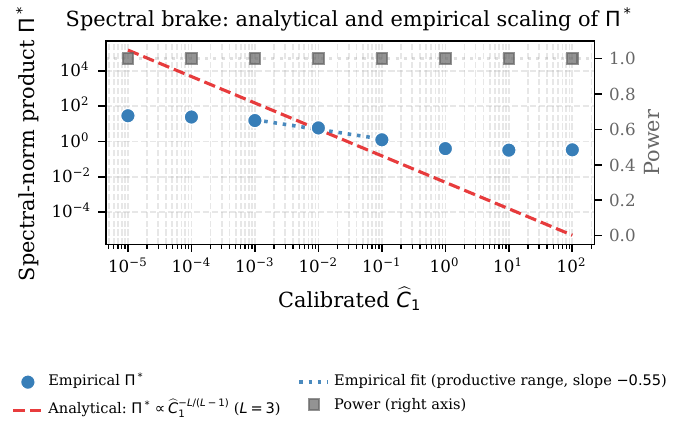}
\caption{Empirical $\Pi^{(T)}$ at convergence (blue circles) decreases with $\widehat C_1$ across seven orders of magnitude (monotone in the productive range $\widehat C_1 \in [10^{-5}, 10]$, with mild floor saturation at $\widehat C_1 = 10^{2}$). The blue dotted line is the least-squares fit to the productive range $[10^{-3}, 10^{-1}]$, with empirical slope $\approx -0.5$. The qualitative observation $\Pi^{(T)} < \infty$ for every $\widehat C_1 > 0$ extends across the entire non-zero sweep. Power (right axis, gray squares) stays $\ge 0.98$ throughout.}
\label{fig:c1-ablation}
\end{figure}

\paragraph{Scope of the ablation.} The sweep measures sensitivity of test power to the \emph{value} of $\widehat C_1$ at fixed architecture, signal, and benign $(d, n, \Delta) = (100, 100, 0.5)$ regime, with three boundary conditions on the regime of validity. First, the sweep varies $\widehat C_1$ at fixed architecture and signal; the class-richness sweep of \S\ref{sec:exp-srm} separately establishes that the \emph{existence} of a positive penalty, i.e., any $\widehat C_1 > 0$, is what prevents plain-max collapse at wide architectures, where Liu and Plain fail. Second, the swept regime is statistically benign: any reasonable kernel attains power near $1$ on $(d=100, n=100, \Delta=0.5)$, so the selector's choice barely matters. In adversarial regimes such as the $(d=20, n=50)$ HDGM cell of Table~\ref{tab:deep-summary}, where Plain falls to $0.14$ and Liu to $0.16$, the deployed power can plausibly depend more strongly on $\widehat C_1$; this ablation does not test that regime. Third, the Cor.~\ref{cor:power-cpmmd} finite-sample power lower bound is parameterized by $C_1$: substituting the worst-case $C_1 \approx 30$ from Theorem~\ref{thm:uci} would render the empirical threshold~\eqref{eq:cor-empirical-threshold} unreachable at our sample sizes, while calibration to $\widehat C_1 \approx 0.008$ makes the certificate non-vacuous. Within these limits, the ablation supplies robustness insurance against estimation error in $\widehat C_1$ in the benign regime.

\subsection{Choice of $n_{\mathrm{cal}}$ and stability}
\label{app:cal-ncal}

In our experiments we use $n_{\mathrm{cal}} = 10$. The seed-to-seed coefficient of variation of $\widehat C_1$ at the calibration anchor $w = 200$ (3 seeds; Table~\ref{tab:c1-arch-sweep}) is $0.7\%$, well within the order-of-magnitude tolerance established by the seven-OOM ablation of \S\ref{app:cal-robustness}, so $n_{\mathrm{cal}} = 10$ is sufficient for stable selection. A formal Bahadur-rate consistency analysis of $\widehat{C}_1 \to \widehat C_1^*$ as $n_{\mathrm{cal}} \to \infty$, where $\widehat C_1^*$ is the population-level invariant defined in~\eqref{eq:c1-target}, is left to future work.

\subsection{Architecture and data-similarity stability of $\widehat C_1$}
\label{app:cal-stability-sweep}

The local-Rademacher-complexity argument of \S\ref{app:cal-invariant} predicts that, for fixed kernel family and data domain, the calibrated $\widehat C_1$ should be approximately stable across architectural variants and mild data perturbations. We test both axes via a controlled sweep on the deep Gaussian setting ($P = \mathcal N(0, I_d)$, $Q = \mathcal N(\Delta\cdot\mathbf 1, I_d)$ with mean-shift parameter $\Delta$; $d = 50$, $n = 100$, $n_{\mathrm{cal}} = 10$, 3 random seeds per cell).

\paragraph{Architecture sweep.} Holding data fixed at $\Delta = 0.5$, we vary MLP width $w \in \{50, 100, 200, 500, 1000\}$ with all other quantities (kernel family, sample size, optimization hyperparameters) held constant.

\begin{table}[h]
\centering
\small
\begin{tabular}{rccccc}
\hline
Width & Mean $\widehat C_1$ & Std & Min & Max & CV (\%) \\
\hline
50   & 0.0135 & 0.0004 & 0.0131 & 0.0140 & 2.9 \\
100  & 0.0123 & 0.0010 & 0.0114 & 0.0137 & 8.1 \\
200  & 0.0089 & 0.0001 & 0.0088 & 0.0090 & 0.7 \\
500  & $\sim 0$ & --- & --- & --- & --- \\
1000 & $\sim 0$ & --- & --- & --- & --- \\
\hline
\end{tabular}
\caption{$\widehat C_1$ vs MLP width on fixed data; factor $\le 1.52\times$ variation across the well-behaved range $w \in [50, 200]$.}
\label{tab:c1-arch-sweep}
\end{table}

We call the range $w \in [50, 200]$ the \emph{well-behaved regime} for plain-max calibration: in this range $\widehat C_1$ varies by a factor of at most $1.52\times$ across architectures, with within-architecture seed-to-seed CV at most $8.1\%$ (Table~\ref{tab:c1-arch-sweep}). At wide architectures ($w \ge 500$), plain-max calibration breaks down: unconstrained ascent on $-\widehat\gamma_{k,u}^2$ drives the spectral product $\Pi$ to unbounded values (no penalty term to brake the ascent in calibration mode), so the proxy denominator $\widetilde G \propto \Pi \cdot \|\mathbf D\|_F/N$ grows to dominate the bounded numerator $\widehat\gamma_{k,u}^2$, driving the calibration ratio $\widehat\gamma_{k,u}^2/\widetilde G$ toward zero. This failure mode is an \emph{unpenalized complexity explosion}, distinct from the variance-collapse failure of $J_{\mathrm{Liu}}$ in App.~\ref{app:variance-collapse} (which arises from a vanishing variance \emph{denominator}, not a growing complexity denominator). Both failure modes affect plain-max under $H_0$ but for different mechanistic reasons; the empirical spectral stabilization (App.~\ref{app:proof-spectral}) prevents both at deployment time because $J_{\mathrm{CP}}$ contains the explicit $-\widehat C_1\widetilde G$ term that calibration mode lacks.

\paragraph{Wide-architecture surrogate-calibration rule.} When the deployed architecture has width $w \ge 500$, the formal calibration procedure of \S\ref{app:cal-procedure} cannot be applied directly. We adopt the following surrogate rule: \emph{calibrate $\widehat C_1$ at the largest well-behaved width within the same kernel family ($w_{\mathrm{surr}} = 200$ in our setup) and reuse the resulting $\widehat C_1$ for all wider deployment architectures.} This carries over because the local Gaussian complexity along the trajectory (which $\widehat C_1^*$ targets) is approximately invariant under architecture refinements that leave the kernel family fixed (Table~\ref{tab:c1-arch-sweep}: $\widehat C_1$ varies $\le 1.52\times$ in $w \in [50, 200]$). The seven-OOM ablation of \S\ref{app:cal-robustness} establishes deployment power is insensitive to $\widehat C_1$ at the order-of-magnitude level, so the surrogate substitution does not degrade test power.

\paragraph{Data-similarity sweep.} Holding architecture fixed at width $= 200$ (the well-behaved regime), we vary the mean-shift $\Delta \in \{0.3, 0.5, 0.7, 1.0\}$.

\begin{table}[h]
\centering
\small
\begin{tabular}{rccccc}
\hline
$\Delta$ & Mean $\widehat C_1$ & Std & Min & Max & CV (\%) \\
\hline
0.3 & 0.0089 & 0.0006 & 0.0080 & 0.0095 & 7.1 \\
0.5 & 0.0089 & 0.0001 & 0.0088 & 0.0090 & 0.7 \\
0.7 & 0.0091 & 0.0006 & 0.0084 & 0.0095 & 6.0 \\
1.0 & 0.0069 & 0.0011 & 0.0055 & 0.0080 & 15.2 \\
\hline
\end{tabular}
\caption{$\widehat C_1$ vs data-shift magnitude on fixed width $= 200$; factor $1.73\times$ across the sweep.}
\label{tab:c1-delta-sweep}
\end{table}

Across the four shift magnitudes, $\widehat C_1$ varies by a factor of $1.73\times$ ($0.0055$ to $0.0095$), with within-seed CV under $8\%$ at $\Delta \in \{0.3, 0.5, 0.7\}$ and $15.2\%$ at $\Delta = 1.0$.

\paragraph{Implication.} Both sweeps support the local-stability claim of \S\ref{app:cal-invariant}: in the well-behaved regime, holding kernel family and data domain fixed, $\widehat C_1$ is determined by the trajectory's local complexity to within a factor of ${\sim}2$. A single calibration at a reference architecture in the well-behaved regime can therefore be reused across deployment instances, as the SRM experiment of \S\ref{sec:exp-srm} does. The wide-architecture breakdown is a regime-of-validity constraint on the calibration procedure alone; deployment uses $J_{\mathrm{CP}}$ ascent, whose spectral brake (App.~\ref{app:proof-spectral}) keeps the trajectory's Lipschitz norm finite even at $w \ge 500$, while the plain-max failure under $H_0$ is specific to the calibration step.

\section{Power consistency at CP-MMD's selected kernel: statement and proof}
\label{app:proof-power-cpmmd}

\subsection{Held-out McDiarmid lemma (prerequisite)}
\label{app:proof-power}

This subsection proves the held-out McDiarmid threshold used inside the proof of Corollary~\ref{cor:power-cpmmd}: at any \emph{fixed} kernel $k$ measurable with respect to the training half, the level-$\alpha$ permutation test on the held-out half rejects with probability $\ge 1 - \delta$ whenever the population MMD exceeds an explicit finite-sample threshold. The lemma below is stated under the balanced sample-split protocol $m_t = n_t$ used throughout \S\ref{sec:experiments}; the unbalanced case admits an analogous bound with $4\nu/\min(m_t, n_t)$ in place of $4\nu/m_t$ at the per-coordinate-swap level, multiplying the held-out McDiarmid threshold by at most a factor of $\sqrt{\rho_*}$.

\paragraph{Conservativeness of the McDiarmid certificate.} The threshold below uses McDiarmid's bounded-differences inequality, which yields a $\Theta(\nu/\sqrt{N_{\mathrm{ho}}})$ rate. Under $H_0$ the unbiased MMD $\widehat\gamma_{k,u}^2$ has the chi-squared-like limiting distribution of a degenerate U-statistic with intrinsic rate $\Theta_p(\nu/N_{\mathrm{ho}})$, established by \citet[Theorem~12]{gretton2012kernel} via the paired-sample U-statistic representation of \citet[Lemma~6]{gretton2012kernel} (with the $m_t = n_t$ balanced regime here giving the cleanest one-sample U-statistic form, whose degeneracy at $H_0$ drives the $1/N_{\mathrm{ho}}$ rate). Consequently, the McDiarmid quantile estimate $c_\alpha = O(\nu\sqrt{\ln(2/\alpha)/N_{\mathrm{ho}}})$ is a known-loose surrogate that overestimates the true permutation null quantile by a $\sqrt{N_{\mathrm{ho}}}$ factor. The bound direction of the certificate is favorable: the corollary requires the empirical MMD to exceed $c_\alpha$, and a larger $c_\alpha$ only makes the certificate more demanding, so the resulting power lower bound is conservative but valid. A sharper analysis invoking Theorem~12 directly would tighten the threshold; we use McDiarmid here for finite-sample distribution-freeness (no kernel-eigenvalue dependence).

Let $\mathbf X^{\mathrm{te}} \sim P^{m_t}$ and $\mathbf Y^{\mathrm{te}} \sim Q^{n_t}$ be held-out i.i.d.\ samples with $N_{\mathrm{ho}} := m_t + n_t$, and let $k$ be a fixed kernel with $|k| \le \nu$. The unbiased estimator $f(\mathbf X^{\mathrm{te}}, \mathbf Y^{\mathrm{te}}) := \widehat\gamma_{k,u}^2(\mathbf X^{\mathrm{te}}, \mathbf Y^{\mathrm{te}})$ has bounded-differences constant $\le 8\nu/N_{\mathrm{ho}}$ in each held-out coordinate. To verify directly: a single $X_l$-coordinate swap perturbs $\frac{1}{m_t(m_t-1)}\sum_{i \ne j} k(X_i, X_j)$ in $2(m_t{-}1)$ pairs, contributing $\le 2\nu/m_t$, and perturbs $\frac{1}{m_t n_t}\sum_{i,j} k(X_i, Y_j)$ in $n_t$ pairs, contributing $\le \nu/m_t$; with the $-2$ coefficient on the cross term, the total per-coordinate change is $\le 4\nu/m_t = 8\nu/N_{\mathrm{ho}}$ for balanced $m_t = n_t$. The $Y$-coordinate calculation is symmetric.

\paragraph{Step 1: concentration of $\widehat\gamma_{k,u}^2$ under $(P, Q)$.}
$\widehat\gamma_{k,u}^2$ is unbiased: $\mathbb E_{(P, Q)}[\widehat\gamma_{k,u}^2] = \gamma_k^2(P, Q)$. By McDiarmid's inequality \citep[Theorem~6.2]{boucheron2013concentration} with bounded-differences $8\nu/N_{\mathrm{ho}}$:
\[
\Pr_{(P, Q)}\!\Bigl(\widehat\gamma_{k,u}^2(\mathbf X^{\mathrm{te}}, \mathbf Y^{\mathrm{te}}) \;\ge\; \gamma_k^2(P, Q) - 8\nu \sqrt{\tfrac{\ln(2/\delta)}{N_{\mathrm{ho}}}}\Bigr) \;\ge\; 1 - \delta/2.
\]

\paragraph{Step 2: concentration of the permutation-null quantile.}
Conditional on the held-out pool, the permutation-null statistic $\widehat\gamma_{k,u}^2(\mathbf X_\pi, \mathbf Y_\pi)$ is a function of the random labeling $\pi$ alone. Each pair-coordinate swap in $\pi$ changes the statistic by at most $8\nu/N_{\mathrm{ho}}$ (the same bounded-differences argument as in Step 1, applied to the relabeling rather than the data). The permutation-null mean equals the unbiased MMD evaluated under random labels, which is $0$ in expectation over $\pi$. Applying McDiarmid's inequality to the random labeling,
\[
\Pr_\pi\!\bigl(\widehat\gamma_{k,u}^2(\mathbf X_\pi, \mathbf Y_\pi) \ge t\bigr) \;\le\; \exp\!\Bigl(-\tfrac{N_{\mathrm{ho}} t^2}{32 \nu^2}\Bigr).
\]
Setting the right-hand side to $\alpha/2$ and solving for $t$ gives
\[
c_\alpha \;\le\; 8\nu\sqrt{\tfrac{\ln(2/\alpha)}{N_{\mathrm{ho}}}},
\]
where $c_\alpha$ is the $(1-\alpha)$-quantile of the permutation-null distribution. The Lehmann--Romano finite-permutation correction of $1/(N_{\mathrm{perm}}+1)$ adds at most $1/(N_{\mathrm{perm}}+1)$ slack \citep[Theorem~15.2.1]{lehmann2005testing}, vanishing for large $N_{\mathrm{perm}}$.

\paragraph{Step 3: combining.}
Union-bounding Steps 1 and 2 (each with failure probability $\le \delta/2$ if we set the perm-null tail to $\alpha/2 \to \delta/2$ for the high-probability statement on $c_\alpha$, or use $\alpha/2$ tied to $c_\alpha$'s defining quantile and $\delta/2$ for the alternative concentration), we have with probability $\ge 1 - \delta$ over $(\mathbf X^{\mathrm{te}}, \mathbf Y^{\mathrm{te}})$:
\[
\widehat\gamma_{k,u}^2(\mathbf X^{\mathrm{te}}, \mathbf Y^{\mathrm{te}}) - c_\alpha \;\ge\; \gamma_k^2(P, Q) - 8\nu\sqrt{\tfrac{\ln(2/\delta)}{N_{\mathrm{ho}}}} - 8\nu\sqrt{\tfrac{\ln(2/\alpha)}{N_{\mathrm{ho}}}}.
\]
The right-hand side is non-negative (and the test rejects, $\phi_k = 1$) whenever
\[
\gamma_k^2(P, Q) \;\ge\; 8\nu\sqrt{\tfrac{\ln(2/\delta)}{N_{\mathrm{ho}}}} + 8\nu\sqrt{\tfrac{\ln(2/\alpha)}{N_{\mathrm{ho}}}} \;\le\; 16\nu\sqrt{\tfrac{\ln(2/\delta) + \ln(2/\alpha)}{N_{\mathrm{ho}}}},
\]
where the last inequality uses $\sqrt a + \sqrt b \le \sqrt{2(a + b)}$. This is the held-out McDiarmid threshold $C\nu\sqrt{(\ln(2/\alpha) + \ln(2/\delta))/N_{\mathrm{ho}}}$ used in Corollary~\ref{cor:power-cpmmd} below with $C = 16$. $\blacksquare$

\subsection{Statement and proof of Corollary~\ref{cor:power-cpmmd}}

\paragraph{Setup.} Adopt Assumption~\ref{assum:kernel}, the alternative $H_1: P \ne Q$, and the sample-split protocol of \S\ref{sec:criterion-deployed}: stratify the global samples (sizes $m$ from $P$ and $n$ from $Q$) into per-class halves of size $m/2$ and $n/2$, forming the training half $\mathcal D_{\mathrm{tr}}$ of size $N_{\mathrm{tr}} := m/2 + n/2$ and the held-out half $\mathcal D_{\mathrm{te}}$ of size $N_{\mathrm{ho}} := m/2 + n/2$. Throughout this appendix, $N$ denotes $N_{\mathrm{tr}}$ (the training-half size); the UCI of Theorem~\ref{thm:uci} is applied to $\mathcal D_{\mathrm{tr}}$ with $(m, n) \to (m/2, n/2)$, leaving $\rho_*$ unchanged for balanced splits. Let $\widehat h \in \mathcal H_T$ be the deployed CP-MMD selector (so that the UCI applies to $\widehat h$), and let $\phi_{k_{\widehat h}}$ be the level-$\alpha$ permutation test at $k_{\widehat h}$ on $\mathcal D_{\mathrm{te}}$ with $N_{\mathrm{perm}}$ permutations. Fix any $\delta, \delta' \in (0, 1)$.

\begin{corollary}[Power consistency at CP-MMD's selected kernel]
\label{cor:power-cpmmd}
On the event that the data realization satisfies
\begin{equation}
\label{eq:cor-empirical-threshold}
\widehat\gamma_{k,u}^2(\widehat h) \;-\; C_1\,G(\mathcal H_T) \;-\; C_2\sqrt{\tfrac{\ln(2/\delta')}{N}} \;\ge\; C\,\nu\,\sqrt{\tfrac{\ln(2/\alpha) + \ln(2/\delta)}{N_{\mathrm{ho}}}},
\end{equation}
the held-out test rejects with probability $\ge 1 - \delta - \delta'$ over the joint randomness of the UCI event $\mathcal E_{\mathrm{tr}}$ on $\mathcal D_{\mathrm{tr}}$ and the held-out sample $\mathcal D_{\mathrm{te}}$, where $C = 16$ suffices.
\end{corollary}

\textbf{Proof.} Let $\mathcal E_{\mathrm{tr}} := \{\sup_{h \in \mathcal H_T} |\widehat\gamma_{k,u}^2(h) - \gamma_k^2(h)| \le B_N(\delta')\}$, which has probability $\ge 1 - \delta'$ over the training half by Theorem~\ref{thm:uci}. On $\mathcal E_{\mathrm{tr}}$, the uniform bound gives, for the deployed selector $\widehat h \in \mathcal H_T$,
\[
\gamma_k^2(\widehat h) \;\ge\; \widehat\gamma_{k,u}^2(\widehat h) - B_N(\delta') \;=\; \widehat\gamma_{k,u}^2(\widehat h) - C_1\,G(\mathcal H_T) - C_2\sqrt{\ln(2/\delta')/N}.
\]
The LHS of~\eqref{eq:cor-empirical-threshold} therefore lower-bounds $\gamma_k^2(\widehat h)$ on $\mathcal E_{\mathrm{tr}}$. Since $\widehat h$ depends only on the training half, the held-out half is i.i.d.\ and independent of $\widehat h$; conditional on $\widehat h$, the held-out McDiarmid lemma of \S\ref{app:proof-power} above applies with $k = k_{\widehat h}$ at confidence $\delta$, controlling the permutation-null quantile by the McDiarmid threshold $C\nu\sqrt{(\ln(2/\alpha)+\ln(2/\delta))/N_{\mathrm{ho}}}$. Whenever~\eqref{eq:cor-empirical-threshold} holds, $\gamma_k^2(\widehat h)$ exceeds this McDiarmid threshold and the held-out test rejects with conditional probability $\ge 1-\delta$. By the tower property and $\Pr(\mathcal E_{\mathrm{tr}}) \ge 1 - \delta'$, the unconditional rejection probability on the threshold event is $\ge (1 - \delta')(1 - \delta) \ge 1 - \delta - \delta'$. $\blacksquare$

\paragraph{Remark (Finite-permutation correction).} The held-out permutation test runs with $N_{\mathrm{perm}}$ permutations rather than the full $\binom{N_{\mathrm{ho}}}{m_t}$ permutation set, so the empirical $(1-\alpha)$-quantile differs from the exact permutation quantile by an additive slack of at most $1/(N_{\mathrm{perm}}+1)$ \citep[Theorem~15.2.1]{lehmann2005testing}. This slack can be absorbed by inflating the $\delta$-tail to $\delta + 1/(N_{\mathrm{perm}}+1)$ throughout, with the rejection probability bound becoming $1 - (\delta + 1/(N_{\mathrm{perm}}+1)) - \delta'$. For $N_{\mathrm{perm}} \ge 200$, this slack is below $0.005$ and is dominated by the leading $\delta$ term.

\paragraph{Scope of the guarantee.} Cor.~\ref{cor:power-cpmmd} requires a non-trivial signal: there must exist some $h \in \mathcal H_T$ with empirical MMD $\widehat\gamma_{k,u}^2(h)$ large enough to clear~\eqref{eq:cor-empirical-threshold}, and the selector $\widehat h$ must reach such an $h$. For pairs $(P, Q)$ with $P \ne Q$ but the empirical MMD on $\mathcal D_{\mathrm{tr}}$ below the threshold, the corollary does not provide a power guarantee, since the kernel class is too poor to distinguish the alternative at the available sample size.

\section{Lipschitz proxy bounds: assumptions, proofs, and feasible set}
\label{app:proxy-bounds}

This appendix gives the closed-form bound used as $\widetilde G(h)$ in each of the three regimes of \S\ref{sec:criterion-deployed}, identifies the feasible set $\mathcal H$ to which the bound applies, and compares the assumptions of each cited result with Assumption~\ref{assum:kernel} (the only assumption Theorem~\ref{thm:uci} requires).

\subsection{Cited results, restated as separate lemmas}
\label{app:proxy-bounds-cited}

We use one classical result: Bartlett's product-of-spectral-norms Lipschitz bound (for the deep regime's $\Pi^* = \prod_j \|W_j\|_2$). We restate it as a self-contained lemma, with the cited reference and an explicit comparison of its hypotheses against our Assumption~\ref{assum:kernel}.

\begin{lemma}[Composition spectral-norm Lipschitz bound, {\citealp[Thm.~1, eq.~(1.1)]{bartlett2017spectrally}}]
\label{lem:bartlett}
Let $h_\theta(x) = W_L\,\phi_{L-1}(W_{L-1}\phi_{L-2}(\cdots \phi_1(W_1 x)\cdots))$ be a feed-forward network with $1$-Lipschitz activations $\phi_j$ satisfying $\phi_j(0) = 0$. Then for every $x, x' \in \mathbb R^{d_0}$,
\[
\|h_\theta(x) - h_\theta(x')\|_2 \;\le\; \Bigl(\prod_{j=1}^{L} \|W_j\|_2\Bigr)\,\|x - x'\|_2,
\]
i.e.\ $h_\theta$ is $\Pi_\theta$-Lipschitz with $\Pi_\theta := \prod_{j=1}^{L}\|W_j\|_2$.
\end{lemma}

\noindent\emph{Hypothesis comparison with Assumption~\ref{assum:kernel}.} Lemma~\ref{lem:bartlett} requires (a) $1$-Lipschitz activations and (b) $\phi_j(0) = 0$. (a) is satisfied by ReLU, LeakyReLU, tanh, and the activations used in our experiments. (b) holds for ReLU and LeakyReLU; for activations with non-zero intercept, the bound holds with a constant offset. Neither (a) nor (b) is implied by Assumption~\ref{assum:kernel}, but both are mild structural conditions on the parametric class of $h$ rather than the kernel $k$, so they are properly understood as a \emph{specification of the feasible set $\mathcal H_{\mathrm{deep}}$}, not as additional hypotheses on the testing problem.

\subsection{Proof of Proposition~\ref{prop:proxy-bound}}
\label{app:proxy-bounds-prop}

We bound $G(\mathcal H_T)$ regime by regime: a direct linear-class calculation for $\mathcal H_{\mathrm{linear}}$, a direct Cauchy--Schwarz + Lipschitz argument for $\mathcal H_{\mathrm{polynomial}}$, and a spectral-norm Gaussian complexity bound \citep{bartlett2017spectrally} for $\mathcal H_{\mathrm{deep}}$. The deep regime carries an additional $\sqrt{\ln N}$ factor that is absorbed into $\widehat C_1$ at calibration.

\begin{proof}[Proof of Proposition~\ref{prop:proxy-bound}]
$G(\mathcal H_T)$ is the feature-class quantity defined in Theorem~\ref{thm:uci}: $G(\mathcal H) = \mathbb E_{\boldsymbol\xi}\sup_{h \in \mathcal H}\tfrac{1}{N}\sum_{i=1}^N \langle \boldsymbol\xi_i, h(\mathbf D_i)\rangle$ with $\boldsymbol\xi_i \in \mathcal U \subset \mathbb R^{d_{\mathrm{out}}}$ i.i.d.\ standard Gaussian vectors and $d_{\mathrm{out}} := \dim(\mathcal U)$. The kernel does not enter (it appears only via $C_1$ in the UCI bound).

\emph{Linear regime} ($\mathcal H_{\mathrm{linear}} = \{x \mapsto Wx : \|W\|_2 \le L^*\}$). Writing $\sum_i \langle \boldsymbol\xi_i, W \mathbf D_i\rangle = \langle W,\, \sum_i \boldsymbol\xi_i \mathbf D_i^\top\rangle_{\mathrm{HS}}$ and taking the supremum over $\|W\|_2 \le L^*$ yields the operator-norm dual $L^* \cdot \|\sum_i \boldsymbol\xi_i \mathbf D_i^\top\|_*$ (nuclear norm). Bounding $\|\cdot\|_* \le \sqrt{d_{\mathrm{out}}} \|\cdot\|_F$ and then applying Jensen with independence,
\[
\mathbb E_{\boldsymbol\xi}\Bigl\|\sum_i \boldsymbol\xi_i \mathbf D_i^\top\Bigr\|_F \;\le\; \Bigl(\sum_i \mathbb E\|\boldsymbol\xi_i\|^2 \cdot \|\mathbf D_i\|^2\Bigr)^{1/2} \;=\; \sqrt{d_{\mathrm{out}}}\,\|\mathbf D\|_F.
\]
Combining gives $G(\mathcal H_{\mathrm{linear}}) \le L^* d_{\mathrm{out}} \|\mathbf D\|_F /N$, which has the form of~\eqref{eq:proxy-bound} with $c_{\mathcal H} = d_{\mathrm{out}}$ absorbed.

\emph{Polynomial regime} ($\mathcal H_{\mathrm{polynomial}} = \{x \mapsto \Psi_p(x)/\sigma : \sigma \in [\sigma_l, \sigma_u]\}$ with $\Psi_p$ the order-$p$ polynomial feature map of Lipschitz constant $L_{\Psi_p}$ on bounded inputs and $\Psi_p(0) = 0$). The supremum over $h \in \mathcal H_{\mathrm{polynomial}}$ ranges over the scalar bandwidth $\sigma$:
\[
G(\mathcal H_{\mathrm{polynomial}}) \;=\; \tfrac{1}{N}\,\mathbb E_{\boldsymbol\xi}\,\sup_{\sigma \in [\sigma_l, \sigma_u]}\,\tfrac{1}{\sigma}\sum_i \langle \boldsymbol\xi_i, \Psi_p(\mathbf D_i)\rangle \;\le\; \tfrac{1}{N\sigma_l}\,\mathbb E_{\boldsymbol\xi}\,\Bigl|\sum_i \langle \boldsymbol\xi_i, \Psi_p(\mathbf D_i)\rangle\Bigr|.
\]
By Jensen's inequality, $\mathbb E_{\boldsymbol\xi}|\sum_i \langle \boldsymbol\xi_i, \Psi_p(\mathbf D_i)\rangle| \le (\sum_i \|\Psi_p(\mathbf D_i)\|^2)^{1/2} = \|\Psi_p(\mathbf D)\|_F$, and the Lipschitz property $\|\Psi_p(x)\| \le L_{\Psi_p}\|x\|$ (which holds for $\Psi_p$ with $\Psi_p(0) = 0$ on bounded inputs) gives $\|\Psi_p(\mathbf D)\|_F \le L_{\Psi_p}\|\mathbf D\|_F$. Combining yields $G(\mathcal H_{\mathrm{polynomial}}) \le L_{\Psi_p}\|\mathbf D\|_F/(N\sigma_l) = L^*\|\mathbf D\|_F/N$ with $L^* = L_{\Psi_p}/\sigma_l$, matching~\eqref{eq:proxy-bound} with $c_{\mathcal H} = 1$.

\emph{Deep regime} ($\mathcal H_{\mathrm{deep}} = \{h_\theta : \prod_j \|W_j\|_2 \le L^*\}$). The empirical Gaussian (and Rademacher) complexity of the spectral-norm ball is bounded by \citet[Thm.~3.1]{bartlett2017spectrally} as
\[
G(\mathcal H_{\mathrm{deep}}) \;\le\; c_{\mathrm{deep}}\, L^*\,\tfrac{\|\mathbf D\|_F}{N}\,\sqrt{\ln N},
\]
where $c_{\mathrm{deep}}$ is an absolute constant depending only on the activation Lipschitz constants and the depth (constant for the architectures used in our experiments). The $\sqrt{\ln N}$ factor is sub-dominant relative to $G(\mathcal H_T)$'s scale and is absorbed into $\widehat C_1$ at calibration.

\emph{Combining.} In each regime, $G(\mathcal H_T) \le c_{\mathcal H} L^* \|\mathbf D\|_F /N$ with $c_{\mathcal H}$ regime-dependent and absorbed into the calibrated $\widehat C_1$, matching~\eqref{eq:proxy-bound}. The Lipschitz constant $L^*$ is the regime-specific quantity tabulated in~\eqref{eq:contraction-bound}: $1/\sigma$ for linear, $L_{\Psi_p}/\sigma$ for polynomial, $\Pi^* = \prod_j \|W_j\|_2$ for deep (Lemma~\ref{lem:bartlett}). $\blacksquare$
\end{proof}

\noindent\emph{From the proposition to the per-iterate proxy.} Proposition~\ref{prop:proxy-bound} bounds $G(\mathcal H_T)$ by $L^*\,\|\mathbf D\|_F/N$, with $L^* = \max_{t \le T} L(h^{(t)})$ the worst-case realized Lipschitz constant along the trajectory. The deployed criterion uses $\widetilde G(h^{(t)}) := L(h^{(t)})\cdot \|\mathbf D\|_F/N$ at iterate $t$. Maximizing $J_{\mathrm{CP}}$ rewards iterates with small $L(h^{(t)})$ and steers the trajectory into a low-Lipschitz region; the empirical spectral stabilization (App.~\ref{app:proof-spectral}) ensures $L(h^{(T)}) \approx L^*$, so the per-iterate value at termination matches the class-level bound. The proxy is thus the per-iterate version of the class bound, made tractable by the stabilization.

\section{Rate-tightness on finite kernel grids: comparison to MMDAgg}
\label{app:proof-mmdagg}

\begin{proposition}[Massart specialization of the parameterized UCI to finite kernel collections]
\label{prop:mmdagg-coverage}
Let $\mathcal K = \{k_1, \ldots, k_B\}$ be any finite collection of $B$ kernels, allowed to be \emph{multi-type}: each $k_j$ may belong to a different family (e.g., Laplacian, Gaussian, Mat\'ern, polynomial) and may use a different bandwidth or shape parameter. Assume each $k_j$ individually satisfies Assumption~\ref{assum:kernel}, write $\nu := \max_{j} \nu_j$ for the common boundedness constant, and let $\rho_*$ be the sample-imbalance ratio of Theorem~\ref{thm:uci}.

\emph{Specialization.} View $\mathcal K$ as a discrete kernel collection indexed by $j \in \{1, \ldots, B\}$ with trivial feature class $\mathcal H = \{\mathrm{id}\}$, so that $G(\mathcal H) = 0$. With this specialization the class-level complexity proxy $\widetilde G(\mathcal K) := \max_j L(k_j)\cdot \|\mathbf D\|_F/N$ is $j$-independent, and the CP-MMD selector reduces to the empirical-MMD argmax,
\[
\widehat\jmath \;:=\; \arg\max_{j \in \{1, \ldots, B\}} J_{\mathrm{CP}}(k_j) \;=\; \arg\max_{j \in \{1, \ldots, B\}} \widehat\gamma_{k_j, u}^2(\mathbf X, \mathbf Y).
\]

\emph{Conclusion.} For any $\delta \in (0, 1)$, with probability at least $1 - \delta$ over the pooled sample $\mathbf D = \mathbf X \sqcup \mathbf Y$,
\begin{equation}
\label{eq:mmdagg-rate}
\gamma_{k_{\widehat\jmath}}^2(P, Q) \;\ge\; \max_{j = 1, \ldots, B} \gamma_{k_j}^2(P, Q) \;-\; 2\,C_2\sqrt{\tfrac{\ln(2B/\delta)}{N}},
\end{equation}
where $C_2 = 4\nu\rho_*$ is the concentration constant of Theorem~\ref{thm:uci} and $\gamma_{k_j}^2(P, Q)$ denotes the population squared MMD between $P$ and $Q$ under kernel $k_j$. The leading $\sqrt{\ln B / N}$ term matches the per-grid separation rate $\Theta(\sqrt{\ln(B/\alpha)/N})$ of MMDAgg \citep[Theorem~9]{schrab2023mmd}.
\end{proposition}

\paragraph{Role of this proposition.} This proposition is a \emph{rate-tightness anchor} for our UCI: specialized to a discrete kernel collection $\mathcal K = \{k_1, \ldots, k_B\}$, our bound recovers the $\Theta(\sqrt{\ln B/N})$ separation rate that \citet[Theorem~9]{schrab2023mmd} establish as minimax-tight on this regime. In the discrete-grid case, the CP-MMD penalty $\widehat C_1 \widetilde G(\mathcal K)$ is $j$-independent, so the CP-MMD selector reduces to plain empirical-MMD argmax, the same selection rule MMDAgg performs internally per kernel. The rate match therefore confirms that our continuous-class UCI does not pay a rate cost on the regime where MMDAgg's tighter analysis is available; any rate advantage of CP-MMD on continuous parametric classes, where MMDAgg cannot be applied, is not bought at the price of looseness on the discrete regime.

\paragraph{Remark (MMDAgg's standard configuration).} MMDAgg's published configuration, a union of Laplacian and Gaussian families across a dyadic bandwidth grid \citep{schrab2023mmd}, is the special case obtained by taking $\mathcal K$ to be that union.

\textbf{Proof.}
\vspace{-1em}
\paragraph{Step 1: Discrete specialization of Theorem~\ref{thm:uci}.}
Apply Theorem~\ref{thm:uci} at each fixed kernel $k_j$ ($j = 1, \ldots, B$) with $\mathcal H = \{\mathrm{id}\}$ at confidence $\delta/B$. Since $G(\{\mathrm{id}\}) = 0$, only the concentration term contributes, giving with probability $\ge 1 - \delta/B$ at each $j$: $|\widehat\gamma_{k_j, u}^2 - \gamma_{k_j}^2(P, Q)| \le C_2\sqrt{\ln(2/(\delta/B))/N}$. Union-bounding over $j$ via Massart's finite-class bound, with probability $\ge 1 - \delta$,
\begin{equation}
\label{eq:mmdagg-uci}
\sup_{j = 1, \ldots, B} \bigl|\widehat\gamma_{k_j, u}^2 - \gamma_{k_j}^2(P, Q)\bigr| \;\le\; C_2\sqrt{\tfrac{\ln(2B/\delta)}{N}} \;=:\; \Delta_{B, \delta},
\end{equation}
where $C_2 = 4\nu\rho_*$ inherits from Theorem~\ref{thm:uci}'s concentration constant.

\paragraph{Step 2: CP-MMD selector tracks the population-best kernel.}
Let $j^* := \arg\max_{j \in \{1, \ldots, B\}} \gamma_{k_j}^2(P, Q)$ denote the population-best kernel index in $\mathcal K$. Apply the UCI event~\eqref{eq:mmdagg-uci} twice, once at the selector $\widehat\jmath$ and once at $j^*$, and use the definition $\widehat\jmath = \arg\max_j \widehat\gamma_{k_j, u}^2$:
\begin{align*}
\gamma_{k_{\widehat\jmath}}^2(P, Q)
&\;\ge\; \widehat\gamma_{k_{\widehat\jmath}, u}^2 - \Delta_{B, \delta}
&& \text{(UCI~\eqref{eq:mmdagg-uci} at $j = \widehat\jmath$)} \\
&\;\ge\; \widehat\gamma_{k_{j^*}, u}^2 - \Delta_{B, \delta}
&& \text{(definition of $\widehat\jmath$ as empirical-MMD argmax)} \\
&\;\ge\; \gamma_{k_{j^*}}^2(P, Q) - 2\Delta_{B, \delta}
&& \text{(UCI~\eqref{eq:mmdagg-uci} at $j = j^*$).}
\end{align*}
Since $\gamma_{k_{j^*}}^2(P, Q) = \max_j \gamma_{k_j}^2(P, Q)$ by definition of $j^*$, this rearranges to~\eqref{eq:mmdagg-rate}. The leading $\sqrt{\ln B/N}$ term in $\Delta_{B, \delta}$ matches MMDAgg's per-grid separation rate \citep[Theorem~9]{schrab2023mmd}: substituting $\delta \mapsto \alpha$ recovers $\Theta(\sqrt{\ln(B/\alpha)/N})$.

\paragraph{Step 3: Extension to infinite $\mathcal H$.}
For a class $\mathcal H$ whose induced data class $\mathcal H(\mathbf D)$ has finite empirical-$L^2$ covering number $N(\eta; \mathcal H(\mathbf D), \|\cdot\|_{L^2(\hat P_N)}) < \infty$ for every $\eta > 0$, Dudley's entropy integral \citep[Theorem~5.22]{wainwright2019high} bounds the Gaussian complexity by
\[
\tfrac{1}{N}\,G(\mathcal H(\mathbf D)) \;\le\; \tfrac{C}{\sqrt N} \int_0^{\nu} \sqrt{\ln N(\eta;\, \mathcal H(\mathbf D),\, \|\cdot\|_{L^2(\hat P_N)})}\;d\eta,
\]
where the upper limit $\nu$ is the kernel boundedness constant (which upper-bounds the diameter of $\mathcal H(\mathbf D)$ in $L^2(\hat P_N)$ via $\|f\|_{L^2(\hat P_N)} \le \|f\|_\infty \le \nu$). The bound is finite under standard entropy conditions (e.g., VC-type or polynomial entropy). The CP-MMD UCI bound therefore extends to infinite classes with no $|\mathcal K|$-dependent multiplicity correction. By contrast, MMDAgg's weighted union bound assigns a multiplicity correction $\alpha / B$ to each kernel; as $B \to \infty$, the per-kernel threshold $\tau_{n,b}(\alpha/B) = O(\nu\sqrt{\ln(B/\alpha)/N}) \to \infty$, and the test becomes trivially conservative. CP-MMD's complexity penalty admits a continuous limit; MMDAgg's does not.\quad$\blacksquare$

\noindent A more comprehensive rate-comparison analysis is left to future work.

\begin{corollary}[Trajectory-adaptive complexity penalty and deployment cost]
\label{cor:adaptive-tax}
\emph{(Trajectory-adaptive penalty.)} For the CP-MMD selector deployed via the trajectory class $\mathcal H_T \subseteq \mathcal H$, the complexity proxy satisfies $\widetilde G(h^{(T)}) \le L^*\,\|\mathbf D\|_F/N$ with $L^* := \max_{t \le T} L(h^{(t)})$ depending on the realized trajectory, whereas MMDAgg pays a fixed $\nu\sqrt{2\ln B/N}$ per kernel uniformly over the grid $\mathcal K$ regardless of which kernels carry signal.

\emph{(Per-test cost.)} At deployment, the per-test cost of CP-MMD is $O(N_{\mathrm{perm}}\cdot N^2)$, independent of search-space size. MMDAgg pays $O(B\cdot N_{\mathrm{perm}}\cdot N^2)$, linear in the grid size $B$. MMD-FUSE pays $O((B + N_{\mathrm{perm}})\cdot N^2)$ via shared distance matrices. Storage is $O(N^2)$ for CP-MMD versus $O(B\cdot N^2)$ for MMDAgg.
\end{corollary}

\textbf{Proof.}

\emph{Trajectory-adaptive penalty.} Proposition~\ref{prop:proxy-bound} bounds $G(\mathcal H_T) \le L^*\,\|\mathbf D\|_F/N$ where $L^* = \max_{t \le T} L(h^{(t)})$. The deployed proxy $\widetilde G(h^{(t)}) := L(h^{(t)})\,\|\mathbf D\|_F/N$ matches this bound at $t = T$ via the empirical spectral stabilization (App.~\ref{app:proof-spectral}). Class-monotonicity of $G$ on $\mathcal H_T \subseteq \mathcal H$ gives $G(\mathcal H_T) \le G(\mathcal H)$, so the trajectory penalty is at most the global one. MMDAgg's per-kernel threshold $\tau_{n,b}(\alpha/B) = \Theta(\nu\sqrt{\ln(B/\alpha)/N})$ is a function of the grid size $B$ and the kernel collection only, while CP-MMD's penalty $\widetilde G(h^{(t)})$ depends on the realized iterate; the two are different design choices for the discrete-grid vs. continuous-class settings, not directly comparable as power claims.

\emph{Per-test cost.} CP-MMD's deployed test runs the level-$\alpha$ permutation test at the single learned kernel $k_{\widehat h}$, requiring $N_{\mathrm{perm}}$ permutations of an $N \times N$ distance matrix at $O(N^2)$ per permutation, total $O(N_{\mathrm{perm}}\cdot N^2)$. MMDAgg evaluates all $B$ candidate kernels for each of $N_{\mathrm{perm}}$ permutations, multiplying the cost by $B$. MMD-FUSE precomputes the $N \times N$ pairwise-distance matrix once across all $B$ kernels (cost $O(B \cdot N^2)$) and reuses it across $N_{\mathrm{perm}}$ permutations (cost $O(N_{\mathrm{perm}} \cdot N^2)$), totaling $O((B + N_{\mathrm{perm}}) \cdot N^2)$~\citep{biggs2023mmdfuse}.\quad$\blacksquare$

\section[Deep ablation: criterion x kernel ansatz factorial]{Deep ablation: criterion $\times$ kernel ansatz factorial}
\label{app:exp-deep-ablation}

To isolate the criterion's effect from the kernel-structure effect in Table~\ref{tab:deep-summary}, we run a $3 \times 2$ factorial. The body's deep results compare three criteria on the same \emph{pure-feature kernel} $k(\phi_\theta(x), \phi_\theta(y))$ (the deep composite $k_h$ of \S\ref{sec:criterion-setup} with $h = \phi_\theta$ and no raw-input mixing); here we cross each criterion (CP-MMD, Liu, Plain) with each of two kernel structures, the pure-feature regime defined above (in scope) and Liu's published $\epsilon$-mixture $k_\epsilon = [(1-\epsilon)\kappa_{\sigma_\kappa}(\phi_w, \phi_w) + \epsilon]\,q_{\sigma_q}$ (out of scope: $\epsilon$ varies inside the kernel rather than inside the feature map). Same MLP architecture as Table~\ref{tab:deep-summary}, mean shift $\Delta = 0.5$, $100$ reps, five $(d, n)$ cells.

\subsection{Heuristic adaptation of $J_{\mathrm{CP}}$ to Liu's $\epsilon$-mixture}
\label{app:exp-deep-ablation-heuristic}

\paragraph{Why an adaptation is needed.} Liu's $\epsilon$-mixture kernel
\[
k_\epsilon(x, y) \;=\; \bigl[(1-\epsilon)\,\kappa_{\sigma_\kappa}(\phi_w(x), \phi_w(y)) + \epsilon\bigr]\,q_{\sigma_q}(x, y),
\]
with parameters $(\phi_w, \sigma_\kappa, \sigma_q, \epsilon)$, has $\epsilon$ varying \emph{inside} the kernel rather than inside a feature map, so $k_\epsilon$ is not of the form $k(h(x), h(y))$ required by Definition~\ref{def:composite-mmd}. Consequently, Theorem~\ref{thm:uci} (UCI) and Proposition~\ref{prop:proxy-bound} (Lipschitz proxy) do not apply directly to this kernel class.

\paragraph{Decomposition of $k_\epsilon$ into two structural sub-kernels.} Expanding the multiplication in $k_\epsilon$,
\[
k_\epsilon(x, y) \;=\; (1{-}\epsilon)\,\underbrace{\kappa_{\sigma_\kappa}(\phi_w(x), \phi_w(y))\,q_{\sigma_q}(x, y)}_{\text{product of feature kernel and raw envelope}} \;+\; \epsilon\,\underbrace{q_{\sigma_q}(x, y)}_{\text{raw envelope}},
\]
which exhibits two structural components blended by $\epsilon$. Neither the product $\kappa\cdot q$ nor $k_\epsilon$ as a whole is in $k_h$ form, so Proposition~\ref{prop:proxy-bound} does not cover them directly.

\paragraph{Heuristic complexity proxy.} We construct $\widetilde G_\epsilon$ by \emph{dropping} the $q_{\sigma_q}$ envelope on the first term (so it becomes a pure feature kernel that Proposition~\ref{prop:proxy-bound} \emph{does} cover) and then taking a convex combination of the two resulting in-scope proxies:
\[
\widetilde G_\epsilon(\phi_w, \sigma_\kappa, \sigma_q, \epsilon) \;:=\; (1-\epsilon)\,\widetilde G_{\mathrm{feat}}(\phi_w, \sigma_\kappa) \;+\; \epsilon\,\widetilde G_{\mathrm{raw}}(\sigma_q),
\]
where each component is the spectral proxy of an explicit in-scope composite sub-kernel:
\begin{itemize}
\item $\widetilde G_{\mathrm{feat}}(\phi_w, \sigma_\kappa) = \tfrac{\Pi}{\sigma_\kappa}\cdot\tfrac{\|\mathbf D\|_F}{N}$ is the \emph{deep-regime} proxy from~\eqref{eq:contraction-bound} applied to the feature sub-kernel $\kappa_{\sigma_\kappa}(\phi_w(\cdot), \phi_w(\cdot))$, viewed as the composite kernel $k(h(x), h(y))$ with base kernel $k = \kappa_1$ (unit-bandwidth Gaussian) and feature map $h(x) := \phi_w(x)/\sigma_\kappa$. The Lipschitz constant of this $h$ on the trajectory class is $L(h) = \Pi/\sigma_\kappa$ with $\Pi := \prod_{j=1}^L \|W_j\|_2$ (Lemma~\ref{lem:bartlett}). $\widetilde G_{\mathrm{feat}}$ is therefore the spectral proxy of $\kappa(\phi_w(\cdot), \phi_w(\cdot))$ \emph{with the $q_{\sigma_q}$ envelope dropped}.
\item $\widetilde G_{\mathrm{raw}}(\sigma_q) = \tfrac{1}{\sigma_q}\cdot\tfrac{\|\mathbf D\|_F}{N}$ is the \emph{linear-regime} proxy from~\eqref{eq:contraction-bound} applied to the raw sub-kernel $q_{\sigma_q}(x, y)$, viewed as the composite kernel $k(h(x), h(y))$ with base kernel $k = q_1$ and linear feature map $h(x) := x/\sigma_q$. The Lipschitz constant is $L(h) = 1/\sigma_q$.
\end{itemize}

\paragraph{Two simplifications.} The heuristic $\widetilde G_\epsilon$ is a Lipschitz-norm-ball bound on the \emph{additive surrogate} $(1-\epsilon)\,\kappa(\phi_w(x), \phi_w(y)) + \epsilon\,q(x, y)$, which differs from the actual mixture $k_\epsilon$ in two respects. (i) The $q_{\sigma_q}$ envelope multiplying $\kappa$ in the actual kernel is dropped at the feature term; the actual product would contribute an additional Lipschitz factor of order $1/\sigma_q$ that the heuristic ignores at the feature term. (ii) The convex combination in $\epsilon$ mirrors the kernel's own $(1-\epsilon)$-vs-$\epsilon$ weighting at the structural level but does not derive from a tight Lipschitz analysis of $k_\epsilon$ as a whole. Both simplifications are absorbed into the calibrated constant $\widehat C_1$ at deployment.

\paragraph{Adapted criterion.} The mixture-aware CP-MMD criterion is
\[
J_{\mathrm{CP}}(\phi_w, \sigma_\kappa, \sigma_q, \epsilon) \;:=\; \widehat\gamma^2_{k_\epsilon, u}(\phi_w, \sigma_\kappa, \sigma_q, \epsilon) \;-\; \widehat C_1\,\widetilde G_\epsilon(\phi_w, \sigma_\kappa, \sigma_q, \epsilon),
\]
optimized jointly over $(\phi_w, \sigma_\kappa, \sigma_q, \epsilon)$ by gradient ascent (with $\epsilon$ parameterized as the sigmoid of an unconstrained scalar to keep $\epsilon \in (0, 1)$).

\paragraph{Calibration of $\widehat C_1$.} We reuse the $\widehat C_1$ calibrated on the in-scope pure-feature kernel form (App.~\ref{app:calibration}) for the mixture's heuristic $\widetilde G_\epsilon$. This reuse is itself a heuristic: the calibration target $\widehat C_1^*$~\eqref{eq:c1-target} depends on the kernel class, and we have no formal stability result across the mixture's $\epsilon$ axis. Empirical insurance comes from App.~\ref{app:cal-robustness}'s seven-OOM ablation, which shows test power is insensitive to $\widehat C_1$ at the order-of-magnitude level.

\paragraph{Scope of formal coverage.} Liu's $\epsilon$-mixture lies outside Definition~\ref{def:composite-mmd}, so the formal claims of Theorem~\ref{thm:uci}, Proposition~\ref{prop:proxy-bound}, and Cor.~\ref{cor:power-cpmmd} do not apply to this kernel class, even with the heuristic adjustment above. Two properties survive:
\begin{itemize}
\item \emph{Type-I validity is exact}: sample-split + held-out permutation test depends only on the held-out half being independent of $\widehat h$, regardless of which kernel class $\widehat h$ is selected from. So the deployed test controls Type-I at level $\alpha$ on Liu's mixture as well as on any in-scope class.
\item \emph{Power and the criterion's advantage are empirical}: the factorial below tests whether the criterion advantage of $J_{\mathrm{CP}}$ over Liu's ratio criterion $J_{\mathrm{Liu}}$, established in scope by \S\ref{sec:experiments}, persists on this out-of-scope kernel class.
\end{itemize}

\subsection{Factorial results}
\label{app:exp-deep-ablation-results}

\begin{table}[h]
\centering\footnotesize
\setlength{\tabcolsep}{4pt}
\renewcommand{\arraystretch}{1.15}
\begin{tabular}{@{}l l ccccc@{}}
\toprule
Criterion & Kernel form & \shortstack{$d{=}2$\\$n{=}200$} & \shortstack{$d{=}20$\\$n{=}200$} & \shortstack{$d{=}50$\\$n{=}100$} & \shortstack{$d{=}100$\\$n{=}100$} & \shortstack{$d{=}20$\\$n{=}50$} \\
\midrule
\textbf{CP-MMD} & pure-feature & $0.81$ & $\mathbf{1.00}$ & $\mathbf{1.00}$ & $\mathbf{1.00}$ & $\mathbf{1.00}$ \\
CP-MMD & Liu mixture & $\mathbf{0.98}$ & $0.97$ & $0.16$ & $0.01$ & $0.54$ \\
\midrule
Liu & pure-feature & $0.31$ & $0.66$ & $0.46$ & $0.73$ & $0.18$ \\
Liu & Liu mixture (published) & $0.93$ & $0.89$ & $0.23$ & $0.00$ & $0.45$ \\
\midrule
Plain & pure-feature & $0.83$ & $1.00$ & $0.54$ & $0.36$ & $0.17$ \\
Plain & Liu mixture & $0.97$ & $0.92$ & $0.23$ & $0.00$ & $0.66$ \\
\bottomrule
\end{tabular}
\caption{Power at mean shift $\Delta = 0.5$, 100 reps, six methods $\times$ five $(d, n)$ cells. Reproduced by \texttt{simulations/run\_liu\_mixture\_ablation.py}; CSV at \texttt{results/liu\_mixture\_ablation/results\_final.csv}.}
\label{tab:ablation-factorial}
\end{table}

Three patterns. \emph{(a) The criterion advantage is structural across kernel ansätze, in scope or out}: within each kernel column (i.e., comparing selection criteria at fixed kernel structure), CP-MMD $\ge$ Plain $\ge$ Liu's ratio criterion in 9 of 10 cells, including on Liu's own mixture kernel that our theory does not formally cover. \emph{(b) The pure-feature kernel structure dominates the mixture form at high $d$}: regardless of criterion, the mixture form's high-$d$ degradation (e.g., $0.98 \to 0.01$ for CP-MMD as $d$ grows from $2$ to $100$) reflects the curse of dimensionality on the raw-input envelope $q$, a property of the \emph{kernel structure} independent of the \emph{criterion}; when both kernel structures are in play, the pure-feature form is the safer default. \emph{(c) Within the mixture column, criterion design still matters}: even where the kernel structure is suboptimal, CP-MMD's penalty-based selection beats Liu's ratio-based selection on the cells where the mixture form retains usable signal ($d=2$ and $d=20, n=200$: $0.98/0.93$, $0.97/0.89$).

\section{Real-data results}
\label{app:exp-real}

We complement the Higgs benchmark of \S\ref{sec:exp-higgs} (where $n \in \{200, 500\}$) with two additional real-data domains, to verify CP-MMD's behavior outside the synthetic-Gaussian regime:
\begin{itemize}
\item \textbf{Tennessee Eastman Process Fault 3/9/15} ($d{=}52$, $n{=}200$): at this small-sample industrial-monitoring regime no learned-kernel approach detects the shifts (CP-MMD, Liu, Plain stay at power $\le 0.08$ across the three faults); only MMDAgg's bandwidth aggregation isolates Fault 15 (power $0.78$), while Faults 3/9 remain at most $0.18$ for every method (\texttt{results/tep/tep\_results.csv}).
\item \textbf{CIFAR-10 vs.\ CIFAR-10.1} ($n{=}200$ image samples after PCA-500 preprocessing): every method we evaluate (CP-MMD, Liu, Plain, Median, MMDAgg, MMD-FUSE) saturates at power $1.00$, so the dataset is uninformative for ranking learned-kernel criteria at this sample size (\texttt{results/cifar\_pca500/cifar\_results.csv}).
\end{itemize}
The Median heuristic wins Higgs at $n{=}200$ (small-sample regime), reflecting the bias-variance trade-off favoring fixed-bandwidth methods at very small $n$.

\section{Limitations}
\label{app:limitations}

Two limitations of the present work merit explicit discussion.

\paragraph{Calibration scope of $\widehat C_1$.} The high-quantile permutation calibration (App.~\ref{app:calibration}) is reliable only at moderate MLP width ($w \in [50, 200]$); at wider widths it degenerates with $\widehat C_1 \to 0$, since the trajectory-quantile estimator absorbs the inflated complexity of overparameterized networks into the data variability rather than the penalty. Empirical spectral stabilization (App.~\ref{app:proof-spectral}, Fig.~\ref{fig:collapse}(d)) keeps training bounded in this regime, but the calibration value transferred from $w{=}200$ no longer reflects the effective complexity of the wider model. In practice, calibration should be performed at a representative width within $[50, 200]$ and reused across architectures of comparable scale; for substantially larger networks, the calibration step should be re-run.

\paragraph{Looseness of the complexity proxy $\widetilde G$.} The Lipschitz proxy $\widetilde G$ (Prop.~\ref{prop:proxy-bound}) is a worst-case Gaussian-complexity surrogate for the spectral-norm complexity that appears in Theorem~\ref{thm:uci}. The three-orders-of-magnitude gap between the worst-case constant $C_1 \approx 30$ and the calibrated $\widehat C_1 \approx 0.008$ is the price of this approximation: the surrogate is correct in shape but conservative by a factor that the calibration absorbs empirically rather than closes in theory. Tightening this gap, via local Rademacher complexity, distribution-dependent envelopes, or PAC-Bayes refinements, is a natural direction for a journal-length theoretical extension.

\end{document}